%% file: main_arxiv.tex
\documentclass{article}
\usepackage{fullpage}

\usepackage[title]{appendix}

\usepackage{algpseudocode}

\usepackage{diagbox}
\usepackage[utf8]{inputenc} 
\usepackage[T1]{fontenc}    
\usepackage{url}            
\usepackage{multirow}
\usepackage{booktabs}       
\usepackage{amsfonts,amssymb,amsthm}       
\usepackage{nicefrac}       
\usepackage{microtype}      
\usepackage{hhline}
\usepackage{graphicx}
\usepackage{booktabs} 

\usepackage{indentfirst}
\usepackage{mathtools}
\usepackage{amsmath}
\usepackage{amssymb}
\usepackage{mathtools}
\usepackage{appendix}

\usepackage[utf8]{inputenc} 
\usepackage[T1]{fontenc}    
\usepackage{hyperref}       
\usepackage{url}            
\usepackage{booktabs}       
\usepackage{amsfonts}       
\usepackage{nicefrac}       
\usepackage{microtype}      
\usepackage{xcolor}         

\usepackage{afterpage}

\usepackage{amsfonts,enumitem,amsmath,amsthm}       
\usepackage{nicefrac}       
\usepackage{microtype}      
\usepackage{xcolor,appendix}         
\usepackage{float}          
\usepackage{graphicx}
\usepackage{subcaption}
\usepackage{amssymb} 
\usepackage[linesnumbered,ruled,vlined]{algorithm2e}

\newtheorem{proposition}{Proposition}

\usepackage{algpseudocode}

\newtheorem{definition}{Definition}
\newtheorem{theorem}{Theorem}

\newtheorem{lemma}{Lemma}

\theoremstyle{plain}
\theoremstyle{definition}

\input{math_commands}

\input{math_commands_specific}

\newcommand{\citet}{\cite}
\newcommand{\citep}{\cite}

\title{\textbf{PromptSplit: Revealing Prompt-Level Disagreement in Generative Models
}}

\date{}

\author{
Mehdi Lotfian$^{1}$,
Mohammad Jalali$^{1}$,
Farzan Farnia$^{1}$ \\
$^{1}$Department of Computer Science and Engineering, The Chinese University of Hong Kong \\
lotfian25@cse.cuhk.edu.hk,
mjalali24@cse.cuhk.edu.hk,
farnia@cse.cuhk.edu.hk
}

\begin{document}
\maketitle

\begin{abstract}
    \input{sec/0_abstract}
\end{abstract}

\section{Introduction}
\input{sec/1_intro}

\section{Related Works}
\input{sec/2_relatedwork}

\section{Preliminaries}
\input{sec/3-prelim}

\section{Method}
\input{sec/4_method}

\section{Numerical Results}
\input{sec/4_numerical}

\section{Conclusion}
\input{sec/5_conclusion}

\bibliographystyle{unsrt}
\bibliography{ref}

\newpage
\appendix
\begin{appendices}
\section{Proofs}
\input{sec/6_appendix_proofs}

\clearpage
\section{Additional Numerical Results}
\input{sec/6_appendix_results}

\end{appendices}

\end{document}

%% file: math_commands.tex
\usepackage{amsmath,amsfonts,bm}

\def\eqref#1{equation~\ref{#1}}

\def\1{\bm{1}}
\def\0{\bm{0}}

\DeclareMathAlphabet{\mathsfit}{\encodingdefault}{\sfdefault}{m}{sl}
\SetMathAlphabet{\mathsfit}{bold}{\encodingdefault}{\sfdefault}{bx}{n}



%% file: math_commands_specific.tex
\usepackage{inputenc}
\usepackage{fontenc}
\usepackage{indentfirst}
\usepackage{cite}
\usepackage{hyperref}
\usepackage{amssymb}
\usepackage{float}
\usepackage{amsmath}
\usepackage{graphicx}
\usepackage{amsthm}

%% file: sec/0_abstract.tex
  Prompt-guided generative AI models have rapidly expanded across vision and language domains, producing realistic and diverse outputs from textual inputs. The growing variety of such models, trained with different data and architectures, calls for principled methods to identify which types of prompts lead to distinct model behaviors. In this work, we propose PromptSplit, a kernel-based framework for detecting and analyzing prompt-dependent disagreement between generative models. For each compared model pair, PromptSplit constructs a joint prompt--output representation by forming tensor-product embeddings of the prompt and image (or text) features, and then computes the corresponding kernel covariance matrix. We utilize the eigenspace of the weighted difference between these matrices to identify the main directions of behavioral difference across prompts. To ensure scalability, we employ a random-projection approximation that reduces computational complexity to $\mathcal{O}(nr^2 + r^3)$ for projection dimension $r$. We further provide a theoretical analysis showing that this approximation yields an eigenstructure estimate whose expected deviation from the full-dimensional result is bounded by $\mathcal{O}(1/r^2)$. Experiments across text-to-image, text-to-text, and image-captioning settings demonstrate that PromptSplit accurately detects ground-truth behavioral differences and isolates the prompts responsible, offering an interpretable tool for detecting where generative models disagree. The
implementation and codebase are publicly available at \url{https://github.com/MehdiLotfian/PromptSplit}.

%% file: sec/1_intro.tex
\begin{figure*}
    \centering
    \includegraphics[width=1\linewidth]{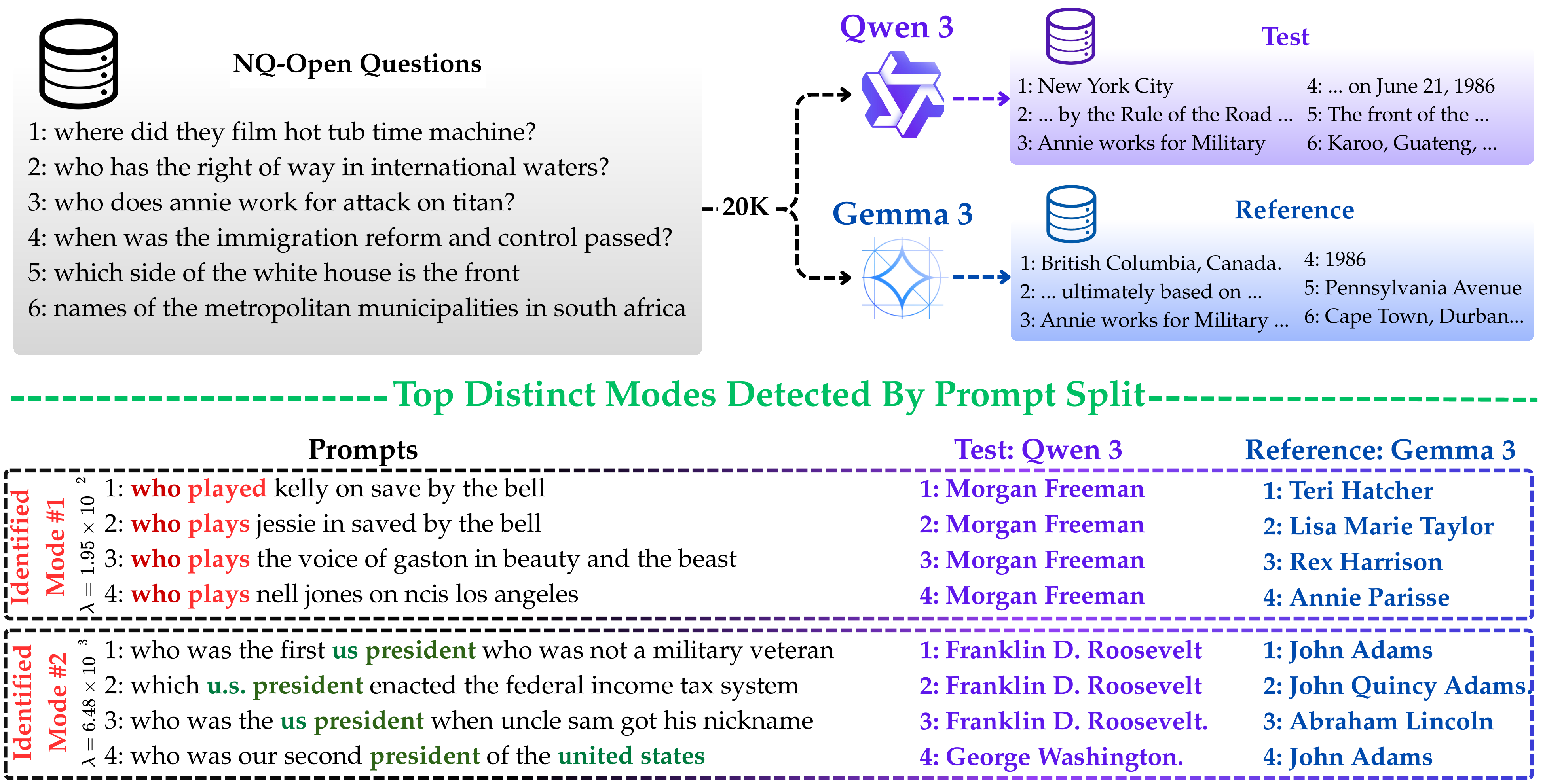}\vspace{-2mm}
    \caption{ Overview of PromptSplit for discovering different types of (prompt,answer) between two models. (a) From NQ-Open questions, we generate outputs from the test model (Qwen3) and reference model (Gemma3). (b) Two high-scoring modes found by PromptSplit: for each mode we show representative prompts and the corresponding model outputs.}
    \label{fig:introduction}
\end{figure*}

Prompt-guided generative models have advanced rapidly across vision and language, enabling high-fidelity synthesis from input text prompts. In computer vision, diffusion-based text-to-image and text-to-video generation systems have demonstrated impressive realism and prompt alignment
\cite{rombach2022ldm,ramesh2022dalle2,saharia2022imagen,singer2022makeavideo,ho2022imagenvideo}. In parallel, large language models (LLMs) have scaled text generation and often serve as language backbones within multimodal pipelines \cite{brown2020gpt3,touvron2023llama,gemini2023,bi2024deepseek,qwenteam2024qwen2,minaee2024llmsurvey}. These advances have produced a rich ecosystem of prompt-guided AI models differing in training data, model architectures, and conditioning strategies.

With such diversity, a central question arises: \emph{how can we systematically compare prompt-guided generative models and determine when their responses differ for an input prompt?} Beyond qualitative demonstrations, quantitative assessment has become a key research focus. While fidelity evaluation metrics such as FID~\cite{heusel2017ttur}, Inception Score~\cite{salimans2016improved}, and CLIP-Score~\cite{hessel2021clipscore} provide useful aggregate indicators of fidelity and alignment, they compress model behavior into single quality-based scores and can obscure prompt-dependent discrepancies across models that may not necessarily stem from the quality of model outputs. For example, a text-to-image model $\mathcal{G}_A$ model may generate a female-individual whenever the prompt contains the word "person", while text-to-image model $\mathcal{G}_B$ always generates the picture of a male individual. While such a difference may not influence the quality of image generation for a single prompt, it may lead to different model responses for certain categories of input prompts.   A more informative comparison should reveal where and how two models behave differently \emph{as a function of the prompt} and specifically identify the prompt categories that lead to divergent output responses by the two models.

Recent work has explored distributional and spectral methods to compare the output spaces of generative models from their samples \cite{zhang2025finc,zhang2024ken}. However, existing formulations generally operate in a prompt-free setting, effectively marginalizing over prompts. Applying such methods directly to prompt-conditioned generation ignores the input prompt and blurs prompt-aware differences into an aggregate distribution. A naive workaround is to analyze each prompt separately by generating many outputs per prompt and running a per-prompt spectral comparison. However, such an approach will be computationally expensive as it requires a significant number of output generations per prompt, which may be unnecessary for the goal of identifying only the prompt categories of the different behavior.

To address this task, we introduce \emph{PromptSplit} (Figure~\ref{fig:introduction}), a prompt-aware unsupervised spectral framework to detect prompt categories leading to different model behaviors, which couples prompts and outputs within a joint representation to attain the goal. To apply PromptSplit, for every generative model, we construct tensor-product embeddings of the prompt and output features (e.g., text and image embeddings in text-to-image models) and then compute the joint kernel covariance matrix of every model. For a pair of models with kernel covariance matrices $C_{X\otimes T}$ and $C_{Y\otimes T}$, we analyze their weighted difference as follows
\begin{equation}
\Lambda_{X, Y|T} \;:=\; C_{X\otimes T} - \,C_{Y\otimes T}.
\end{equation}
 We highlight that the principal eigenvalues and eigenvectors of $\Lambda_{X, Y|T}$ can reveal the prompt clusters resulting in different output structures by the models. To compute the eigenvectors of the above matrix, we apply kernel trick and show that the eigendirections of the above matrix is in one-to-one correspondence to those of the following block kernel matrix:
\begin{equation*}
\mathbf{K}_{X,Y|T} \;=\;
\begin{bmatrix}
K_{XX}\odot K_{TT} & K_{XY} \odot K_{TT}\vspace{2mm} \\
-K_{YX}\odot K_{TT} & -K_{YY} \odot K_{TT}
\end{bmatrix},
\end{equation*}
where $\odot$ denotes the elementwise matrix product and each $K_{TT'},K_{XX},K_{XY}$ represents kernel similarities between prompts or outputs across the two models. This joint structure enables spectral comparison of model responses while explicitly accounting for prompt information, allowing PromptSplit to identify prompt categories that differentiate the two generation models.

Subsequently, we discuss that a direct computation of the eigenvectors of $\mathbf{K}_\Delta$ will lead to a cubically growing complexity $\mathcal{O}\bigl((m+n)^3\bigr)$ in the number of samples, limiting the application of the algorithm to $m,n$ values of at most a few tens of thousands. To apply PromptSplit for significantly larger sample sizes, which would be necessary to ensure its proper convergence, we introduce a random-projection reduction of the joint (tensor) features to a target dimension $r$, reducing per-sample cost to $\mathcal{O}(r(d_t+d_x))$ for prompt and output embedding dimensions $d_t,d_X$, and overall spectral decomposition complexity to $\mathcal{O}\bigl((m+n)r^2+r^3\bigr)$. We further show that this approximation achieves an expected eigenspace deviation bounded by $\mathcal{O}(1/r^2)$, enabling efficient spectral analysis with bounded random projection dimension $r$ values.

We evaluate PromptSplit across text-to-image models and text-to-text (LLM) comparisons. In synthetic settings with known prompt-wise differences, PromptSplit accurately recovers the responsible prompt categories. On real systems, including latent-diffusion and diffusion-transformer models such as Stable Diffusion, Kandinsky, and PixArt—PromptSplit reveals prompt families where responses diverge in style, composition, and alignment, complementing aggregate metrics with interpretable prompt-level disagreement maps \cite{rombach2022ldm,razzhigaev2023kandinsky,arkhipkin2024kandinsky3,chen2023pixartalpha}.In summary, this work (i) formulates prompt-level model comparison as a joint prompt–output spectral problem, (ii) introduces PromptSplit for analyzing the eigenspectrum of joint kernel covariance differences, (iii) provides a scalable random-projection approximation with $\mathcal{O}(1/r^2)$ theoretical accuracy given projection dimension $r$, and (iv) numerically demonstrates the analysis of prompt-dependent disagreement across synthetic and real prompt-guided generative models.

%% file: sec/2_relatedwork.tex
\textbf{Evaluation of generative models.}
Classical metrics for evaluating generative models include IS~\cite{salimans2016improved}, FID~\cite{heusel2017ttur}, and KID~\cite{binkowski2018kid}, while precision–recall metrics~\cite{kynkaanniemi2019precisionrecall} and density/coverage~\cite{naeem2020density} separately measure fidelity and diversity. Spectral metrics such as Vendi~\cite{friedman2022vendi}, RKE~\cite{jalali2023rke}, and KEN~\cite{zhang2024ken} incorporate eigenvalue distributions of kernel similarities to capture global diversity and novelty. CLIPScore~\cite{hessel2021clipscore} evaluates text–image alignment using CLIP embeddings~\cite{radford2021clip}, and recent analyses~\cite{stein2023exposing} show that DINOv2 features~\cite{oquab2023dinov2} often yield more reliable evaluations than Inception features. These metrics provide scalar fidelity/diversity summaries, whereas PromptSplit uses embeddings only to build joint kernels and performs prompt-conditioned spectral comparison of two models.

\textbf{Comparisons of generative models.}
Spectral and kernel-based approaches compare generative models by analyzing eigenspectra or kernel embeddings of their outputs. The spectral methods in ~\cite{zhang2024ken,zhang2025finc,ospanov2024fkea} detects novelty relative to a reference distribution. Beyond generative models, dataset-level comparison frameworks explain how two datasets differ, either via interpretable prototype and influential-example explanations~\cite{babbar2025datasets} or via transport maps and counterfactuals for image-based distribution shifts~\cite{Kulinski_2022_CVPR}, and survey work provides a broad taxonomy of dataset similarity measures including distance-, kernel-, and embedding-based criteria~\cite{stolte2024datasetsimilarity}. Compared to these unconditional or dataset-level methods, PromptSplit directly targets \emph{prompt-conditioned} comparison by forming joint prompt–output embeddings and performing spectral analysis on prompt-dependent covariance differences.

\textbf{Interpretability for generative models and embeddings.}
Network Dissection~\cite{bau2017netdissect} and GAN Dissection~\cite{bau2019gandissect} analyze internal units of CNNs and GANs to identify semantic concepts, while GANSpace~\cite{harkonen2020ganspace} uncovers interpretable latent directions through PCA in latent space. Concept activation vectors (TCAV)~\cite{kim2018tcav} assign user-defined concepts to model directions. These methods explain latent or hidden representations \emph{within} a single model. PromptSplit differs by analyzing disagreement between \emph{two models} through eigenspaces of prompt–output kernel covariance differences, without accessing internal activations.

\textbf{Random projection of feature embeddings.}
Random features~\cite{rahimi2007random} approximate shift-invariant kernels with low-dimensional embeddings, enabling scalable kernel learning. Approximate kernel $k$-means~\cite{chitta2011akk} and explicit polynomial feature maps~\cite{pham2013explicit} accelerate clustering and large-scale learning, and Fourier features~\cite{tancik2020fourier} are widely used to capture high-frequency structure in neural fields. PromptSplit uses a model-agnostic linear projection of joint tensor-product embeddings to reduce the cost of eigen-decomposition while preserving disagreement directions.

%% file: sec/3-prelim.tex
\subsection{Kernel Matrices and Covariance Operators}

A kernel $k:\mathcal{X}\times\mathcal{X}\to\mathbb{R}$ is defined to be a symmetric positive semi-definite (PSD) function that admits a feature map $\phi:\mathcal{X}\to\mathcal{H}$ into a reproducing kernel Hilbert space (RKHS) $\mathcal{H}$ such that
\[
k(x,y)=\langle \phi(x),\phi(y)\rangle_{\mathcal{H}}.
\]
Given samples $\{x_i\}_{i=1}^n$, the empirical kernel (Gram) matrix is defined as $
K_{XX}=\bigl[k(x_i,x_j)\bigr]_{i,j=1}^n\in\mathbb{R}^{n\times n}$, and for two sets $\{x_i\}_{i=1}^n$ and $\{y_j\}_{j=1}^m$, the cross-kernel matrix is $K_{XY}=\bigl[k(x_i,y_j)\bigr]_{i,j}$.

The kernel covariance operator associated with a random variable $X\sim P_X$ is defined by
\[
C_X=\mathbb{E}_{X\sim P_X}\!\big[\phi(X)\phi(X)^\top\big].
\]
We denote the empirical kernel covariance operator with $\widehat{C}_X=\tfrac{1}{n}\sum_{i=1}^n\phi(x_i)\phi(x_i)^\top$. Note that the nonzero eigenvalues of $\widehat{C}_X$ coincide with those of $\tfrac{1}{n}K_{XX}$,
since $\tfrac{1}{n}K_{XX} = \frac{1}{n}\Phi_X^\top \Phi_X$ and $\widehat{C}_X = \frac{1}{n}\Phi_X \Phi_X^\top$, for matrix $\Phi_X = \bigl[\phi(x_1); \cdots; \phi(x_n) \bigr]^\top \in \mathbb{R}^{n\times d}$, where the matrix multiplication order is flipped, preserving the non-zero eigenvalues. 
Assuming a normalized kernel satisfying $k(x,x)=1$ for every $x\in\mathcal{X}$, we note that the eigenvalues of the normalized kernel matrix $\tfrac{1}{n}K_{XX}$ will be all non-negative and they sum up to one as $\mathrm{Tr}\bigl(\tfrac{1}{n}K_{XX}\bigr) = 1$.  

We call a kernel function $k$ shift invariant if there exists a function $\kappa:\mathbb{R}^d\rightarrow\mathbb{R}$ such that for every $x,y\in\mathbb{R}^d: k(x,y)= \kappa(x-y)$. For shift-invariant kernels of the form $k(x,y)=\kappa(x-y)$ on $\mathbb{R}^d$, Bochner’s theorem proves that the Fourier transform $\widehat{\kappa}:\mathcal{R}^d\rightarrow \mathbb{R}$ is a valid probability density function (PDF), where the Fourier transform is defined as
\begin{equation*}
   \widehat{\kappa}(\omega) := \frac{1}{(2\pi)^d}\int_{\mathbb{R}^d} \kappa(x) \exp\bigl(-i\langle \omega , x\rangle\bigr)\mathrm{d}x, 
\end{equation*}
then we have $
k(x,y)={\mathbb{E}}_{\omega\sim\widehat{\kappa}}\bigl[\exp(i\langle \omega , x-y\rangle )\bigr]$.
The Random Fourier Features (RFF) \cite{rahimi2007random,sutherland2015error} approximate such kernels via Monte Carlo sampling. Drawing $\boldsymbol{\omega}_1,\ldots,\boldsymbol{\omega}_r\stackrel{\text{i.i.d.}}{\sim}\widehat{\kappa}$, we define the following proxy RFF feature map:
\[
\varphi_r(x)=\frac{1}{\sqrt{r}}\big[\cos(\boldsymbol{\omega}_\ell^\top x),\sin(\boldsymbol{\omega}_\ell^\top x)\big]_{\ell=1}^r,
\]
which satisfies $\mathbb{E}[\varphi_r(x)^\top\varphi_r(y)]=k(x,y)$. Replacing $\phi(x)$ with $\varphi_r(x)$ yields a low-dimensional approximation
\[
\widehat{C}_X\approx \frac{1}{n}\sum_{i=1}^n \varphi_r(x_i)\varphi_r(x_i)^\top\in\mathbb{R}^{2r\times2r}.
\]

\subsection{Tensor-Product Kernels and Hadamard Joint Kernel Matrices}

Let $T$ and $X$ denote the random variables of an input prompt and generated output, equipped with feature maps  
$\phi_T:\mathcal{T}\to\mathbb{R}^{d_t}$ and $\phi_X:\mathcal{X}\to\mathbb{R}^{d_x}$.  
We define the joint tensor-product feature map as
\[
\phi_{\otimes}(t,x)=\phi_T(t)\otimes \phi_X(x)\in\mathbb{R}^{d_t}\otimes\mathbb{R}^{d_x},
\]
where $\otimes$ denotes the tensor product of Hilbert spaces, which is for vectors $t=[t^{(1)},\ldots , t^{(d_t)}]\in\mathbb{R}^{d_t}$ and $x=[x^{(1)},\ldots , x^{(d_x)}]\in\mathbb{R}^{d_x}$, is defined as:
\begin{equation*}
    t\otimes x \!=\! \bigl[t^{(1)} x^{(1)},. , t^{(1)} x^{(d_x)},..,t^{(d_t)}x^{(1)},. , t^{(d_t)}x^{(d_x)} \bigr] \!\in\!\mathbb{R}^{d_xd_t}
\end{equation*}
This representation captures multiplicative interactions between prompt and output embeddings and induces the product kernel
\begin{align*}
k_{\otimes}\bigl([t,x],[t',x']\bigr)
&\, =\,\bigl\langle \phi_{\otimes}(t,x),\phi_{\otimes}(t',x')\bigr\rangle
\\
&
\, =\, k_T(t,t')\cdot k_X(x,x').
\end{align*}
Given empirical samples $\{(t_i,x_i)\}_{i=1}^n$, the empirical joint kernel covariance is
\[
C_{X\otimes T}
=\frac{1}{n}\sum_{i=1}^n \phi_\otimes(t_i,x_i)\phi_\otimes(t_i,x_i)^\top.
\]

%% file: sec/4_method.tex
We are given two prompt–response datasets produced by two generative systems,
\[
\mathcal{D}_X=\{(t_i,x_i)\}_{i=1}^{n},\qquad
\mathcal{D}_Y=\{(t'_j,y_j)\}_{j=1}^{m}.
\]
Each prompt $t$ and output $z\!\in\!\{x,y\}$ is embedded and mapped to RKHSs
via feature maps $\phi_T:\mathcal{T}\!\to\!\mathcal{H}_T$ and
$\phi_X,\phi_Y:\mathcal{X}\!\to\!\mathcal{H}_X$ with normalized kernels $k_\mathcal{T}:\mathcal{T}\times\mathcal{T}\to \mathbb{R}$ and $k_\mathcal{X}:\mathcal{X}\times\mathcal{X}\to \mathbb{R}$ 
satisfying $k_T(t,t)=k_X(x,x)=1$ for every $t\in \mathcal{T},\, x\in\mathcal{X}$.  
Our goal is to identify and interpret where, in the joint prompt–output
space, the two systems differ.  When prompt distributions between the two datasets match $P_T = P_{T'}$, the
analysis from \cite{zhang2024ken}suggests that the kernel matrices of the collected data can reveal the differences between the two conditional distributions $P_{X\mid T}$ and $P_{Y\mid T}$.

\subsection{PromptSplit via Joint Kernel Covariance Difference}

Consider the joint tensor feature $\phi_\otimes(t,x)=\phi_T(t)\otimes\phi_X(x)$. Then, the empirical joint covariances of the two datasets $\mathcal{D}_X$ and $\mathcal{D}_Y$ are defined as
\begin{align*}
\widehat{C}_{T\otimes X}&=\frac{1}{n}\sum_{i=1}^n
   \phi_\otimes(t_i,x_i)\phi_\otimes(t_i,x_i)^{\top},\\
\widehat{C}_{T\otimes Y}&=\frac{1}{m}\sum_{j=1}^m
   \phi_\otimes(t'_j,y_j)\phi_\otimes(t'_j,y_j)^{\!\top}.
\end{align*}

\begin{definition} 
 The covariance–difference operator between $\mathcal{D}_X$ and $\mathcal{D}_Y$ with hyperparameter $\eta>0$ is defined as
\begin{equation}
\widehat{\Lambda}_{X,Y|T}
=\widehat{C}_{T\otimes X}-\eta\,\widehat{C}_{T\otimes Y},
\end{equation}
\end{definition}
As demonstrated in \cite{zhang2024ken}, the eigendirections of the above matrix can reveal the differently expressed modes between the distributions $P_{T,X}$ and $P_{T',Y}$. We highlight that when the prompt marginal distributions $P_T$ and $P_T'$ of the two datasets are the same, the differences between the (prompt,output) joint distributions will reveal the differences between the conditional models $P_{X|T}$ and $P_{Y|T}$, that is precisely the aim of our comparative analysis. Therefore, in our analysis, we aim to efficiently compute the principal eigendirections of $\widehat{\Lambda}_{X,Y|T}$ for a sufficiently large dataset size $n$.

\begin{algorithm}[t]
\caption{PromptSplit: Kernel-based Formulation}
\label{alg:promptsplit}

\KwIn{Datasets $\mathcal{D}_X=\{(t_i,x_i)\}_{i=1}^{n}$,
$\mathcal{D}_Y=\{(t'_j,y_j)\}_{j=1}^{m}$;
kernels $k_T,k_X$; parameter $\eta>0$;
eigenpair number $R$.}

\KwOut{Positive eigvalues $\{\lambda_r\}$, eigenfunctions $\{v_r\}$.}

Build kernel blocks:
$K_{TT},K_{XX}$ on $\mathcal{D}_X$;
$K_{T'T'},K_{YY}$ on $\mathcal{D}_Y$;
and cross-blocks $K_{TT'},K_{XY}$.

Form $K_{X,\eta Y|T}$ as in~\eqref{eq:kblock}.

Compute the top $R_+$ positive eigenpairs
$\{(\lambda_r,u_r)\}$ with $u_r=[u_{1:n}^{(r)};u_{(n+1):(n+m)}{(r)}]$.

Construct eigenfunctions
$v_r=\sum_{i=1}^n u^{(r)}_i\phi_X(t_i,x_i)
     +\sum_{j=1}^m u^{(r)}_{n+j}\phi_X(t'_j,y_j)$.
     
\end{algorithm}

First, we observe that for the prompt and output feature dimensions $d_t,d_x$, the dimension of matrix $\widehat{C}_{X\otimes T}$ will be $d_td_x\times d_td_x$. As the standard embedding dimensions for input prompt and output visual data are usually lower bounded by several hundreds, e.g. $512$ for CLIP embeddings, performing the eigendecomposition of $\widehat{\Lambda}_{X,Y|T}$ with even a simple linear kernel, would require the eigendecomposition of a matrix with at least a few hundreds of thousands rows, which would be infeasible on standard CPU and GPU processors.

To address the computatioanl challenge, we first apply the kernel trick and formulate a kernel matrix of size $2n\times 2n$ for $n$ samples that share the eigenspectrum with $\widehat{\Lambda}_{X,Y|T}$. To do this, let $K_{TT},K_{XX}$ be prompt/output kernels on $\mathcal{D}_X$, and
$K_{T'T'},K_{YY}$ prompt/output kernels on $\mathcal{D}_Y$, and
$K_{TT'},K_{XY}$ be the cross-kernels.
Considering the product kernel
$k_\odot\bigl([t,x],[t',x']\bigr)=k_\mathcal{T}(t,t')\cdot k_\mathcal{X}(x,x')$, we define the following kernel matric $K_{X,\eta Y|T}$ 
\begin{equation}
K_{X,\eta Y|T}=\hspace{-1.5mm}
\begin{bmatrix}
\frac{1}{n}K_{TT}\odot K_{XX}
  & \frac{1}{\sqrt{nm}}K_{TT'}\odot K_{XY}\vspace{2mm}\\
-\frac{\eta}{\sqrt{nm}}K_{TT'}^\top\odot K_{XY}^{\top}
  & -\frac{\eta}{m}K_{T'T'}\odot K_{YY}
\end{bmatrix}.
\label{eq:kblock}
\end{equation}

\begin{proposition}\label{prop:hadamard}
The matrices $\widehat{\Lambda}_{X,Y|T}$  and $K_{X,\eta Y|T}$ share the same non-zero eigenvalues. Also,  
for every $K_{X,\eta Y|T}$'s eigenvector $u=[u_{1:n};u_{(n+1):(n+m)}]$ of $K_{X,\eta Y|T}$
with non-zero eigenvalue $\lambda$, then the following $v$ is the eigenvector of $\widehat{\Lambda}_{X,Y|T}$ for the same eigenvalue $\lambda$:
\[
v=\sum_{i=1}^{n}v_i\phi_\otimes(t_i,x_i)
  +\sum_{j=1}^{m}v_{n+j}\phi_\otimes(t'_j,y_j)
\]
\end{proposition}
\begin{proof}
We defer the proof to the Appendix.
\end{proof}
The above proposition shows that given $n,m$ samples in the two datasets, the cost of the eigendecomposition of $\widehat{\Lambda}_{X,Y|T}$ will grow at most as $\mathcal{O}\bigl((m+n)^3\bigr)$, where the upper-bound is independent of the kernel feature map size, as long as the kernel function can be efficiently evaluated.  

\begin{algorithm}[t]
\caption{Random Projection PromptSplit}
\label{alg:rp-promptsplit}
\KwIn{Datasets $\mathcal{D}_X,\mathcal{D}_Y$;
explicit features $\phi_T,\phi_X $
(or RFF maps); RP size $r$;
parameter $\eta>0$; number of eigenpairs $R$.}
\KwOut{Reduced-dimensional eigenpairs
$\{(\widehat{\lambda}_r,\widehat{w}_r)\}$.} 
Draw Gaussian matrices
$R_T\!\sim\!\mathcal{N}(0,1)^{d_t\times r}$ and
$R_X\!\sim\!\mathcal{N}(0,1)^{d_x\times r}$;
Set $R=\frac{1}{r}(R_T \otimes R_X)$.

Compute sketched joint features:
$\widetilde{\phi}_{r}(t_i,x_i)
  =R_T\phi_T(t_i)\odot R_X\phi_X(x_i) \quad \text{and} \quad
\widetilde{\phi}_{r}(t'_i,y_i)
  =R_T\phi_T(t'_i)\odot R_X\phi_X(y_i)$.

  Form
$\widehat{\Lambda}^{(r)}
  =\frac{1}{n}\!\sum_i\!\widetilde{\Phi}_{r,X}\widetilde{\Phi}_{r,X}^{\!\top}
   -\frac{\eta}{m}\!\sum_j\!\widetilde{\Phi}_{r,Y}\widetilde{\Phi}_{r,Y}^{\!\top}$.

   Compute positive eigenpairs
$(\widehat{\lambda}_r,\widehat{w}_r)$ of $\widehat{\Lambda}^{(r)}$.
\end{algorithm}

\subsection{Random Projection for Scalable PromptSplit}
As discussed earlier, we can perform the eigendecomposition of $K_{X,\eta Y|T}$, with compute cost of 
$\mathcal{O}\bigl((n+m)^3\bigr)$ to compute the eigenspace of the target kernel covariance operator difference $\widehat{\Lambda}_{X,Y|T}$. However, when the dataset sizes grow beyong few tens of thousands, this approach becomes computationally infeasible.

To reduce the computational costs, in this section we propose a joint Gaussian random projection that approximately preserves the
kernel covariance difference geometry while reducing the dimensionality of the feature map.
Let $R_T\in\mathbb{R}^{d_T\times r}$,
$R_X\in\mathbb{R}^{d_X\times r}$ with i.i.d. $\mathcal{N}(0,1)$
entries and then define
\[
\widetilde{\phi}_r(x,t)=\frac{r}{\sqrt{d_T d_X}}\,R_T\phi_\mathcal{T}(t)\odot R_X \phi_\mathcal{X}(x)
   \in\mathbb{R}^{r}
\]
Note that the computational cost of computing $\widetilde{\phi}_r(x,t)$ will be $\mathcal{O}\bigl(r(d_x+d_t)\bigr)$. 
Then, we propose to consider the kernel covariance  difference operator with the joint feature map $\widetilde{\phi}_r$ to define:
\begin{equation}
\widetilde{\Lambda_r}_{X,Y|T}
=\widehat{C}_{\widetilde{\phi}_r(X,T)}-\eta\,\widehat{C}_{\widetilde{\phi}_r(Y,T)} \in \mathbb{R}^{r\times r}.
\end{equation}
Note that the total computational cost of computing and performing the eigen decomposition of $\widetilde{\Lambda_r}_{X,Y|T}$ wil be $\mathcal{O}\bigl((m+n)r(d_x+d_t)+r^3\bigr)$, which will grow linearly with $(m+n)(d_x+d_t)$ for a properly bounded random projection size $r$.

Furthermore, we note that in the case of a shift-invariant kernel, the above random projection can be unified with the mapping to the random Fourier feature space with the following definition of $\widetilde{\phi}_r(t,x)$ (for an even integer $r$) given the Fourier features $\omega_{x,1},\ldots ,\omega_{x,r/2}\sim \widehat{\kappa}_x$ and  $\omega_{t,1},\ldots ,\omega_{t,r/2}\sim \widehat{\kappa}_t$ for the output and prompt kernels:

\begin{figure*}[t]
    \centering
    \includegraphics[width=1\linewidth]{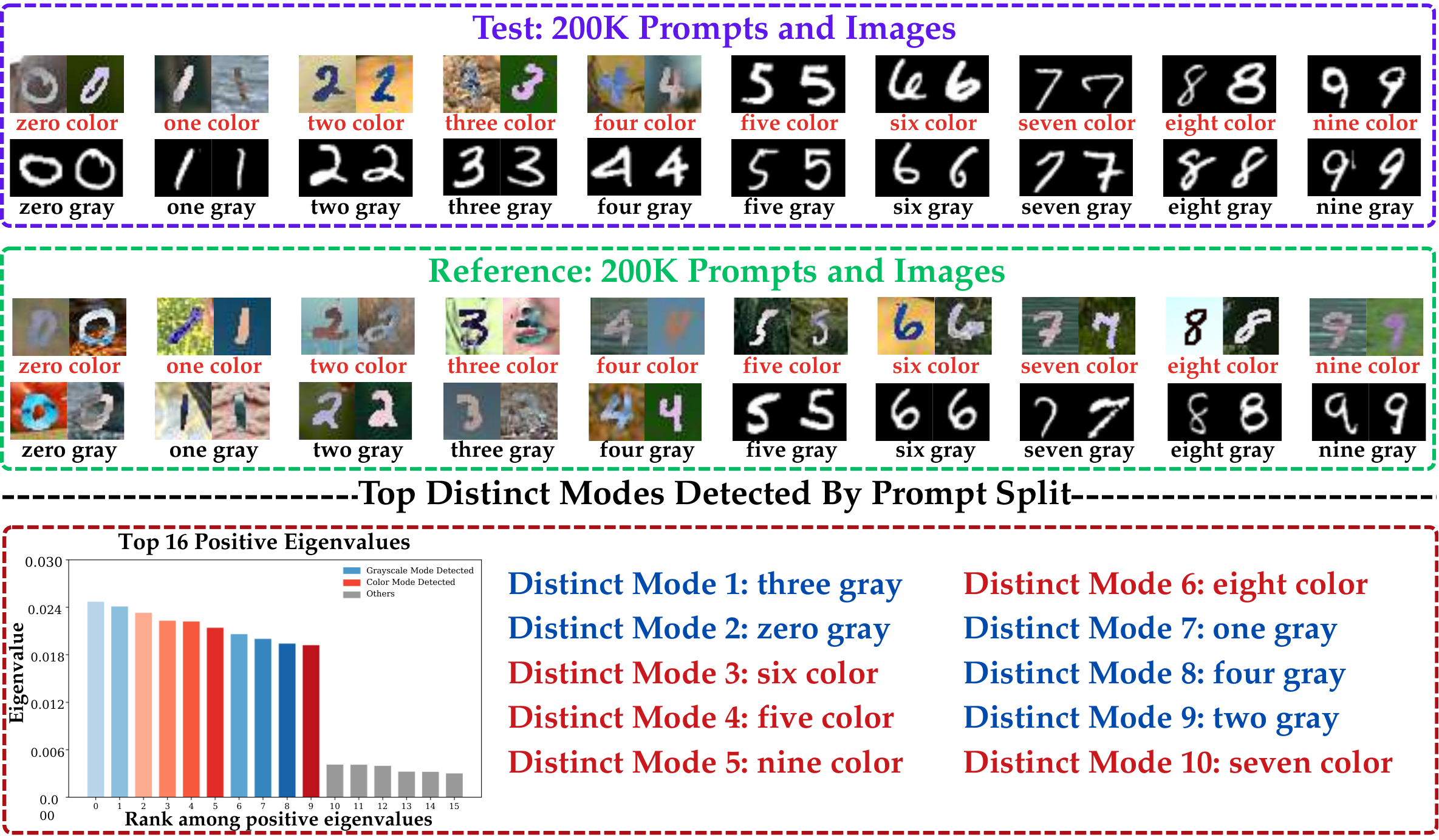}
    \caption{PromptSplit identified top clusters of prompts with distinct images. Top: 20 clusters of prompts with sample images for test and reference dataset. Bottom: Top 16 eigenvalues showing top 10 disagreement causing prompts.}
    \label{fig:mnist}
\end{figure*}

\begin{equation*}
 \widetilde{\phi}_r(t,x) = \bigl[\cos(\omega_{t,i}^\top t + \omega_{x,i}^\top x),\sin(\omega_{t,i}^\top t + \omega_{x,i}^\top x)\bigr]_{i=1}^{r/2} \in\mathbb{R}^{r}.   
\end{equation*}
In the following, we prove that the eigendirections of the random projection $\widetilde{\Lambda_r}_{X,Y|T}$ will lead to an $O(\frac{1}{r^2})$-accurate approximation of the eigendirections of the  matrix ${\Lambda_r}_{X,Y|T}$.  

\begin{theorem}
\label{thm:rp}
Consider the kernel covariance difference matrix $K_{X,\eta Y|T}$ and the proxy kernel covariance difference matrix $\widetilde{K_r}_{X,\eta Y|T}$ of the random projection approach with dimension $r$. Then, for every $\delta>0$, the following holds with probability at least $1-\delta$,
\begin{align*}
&\bigl\|\boldsymbol{\lambda}(K_{X,\eta Y|T})-\boldsymbol{\lambda}(\widetilde{K_r}_{X,\eta Y|T})\bigr\|_2 \le
   \sqrt{\frac{8+8\eta^2}{r}}
\bigl(1+\sqrt{2\log\tfrac{1}{\delta}}\bigr). 
\end{align*}
\end{theorem}
\begin{proof}
We defer the proof to the Appendix.
\end{proof}

\section{PromptSplit Guidance for Text-Guided Diffusion Models}
\label{sec: maintext guidance}
As discussed in the previous section, PromptSplit can identify differences in text-guided generative models. One application of this framework is to use it for guiding conditional diffusion models to align with a reference dataset, where we specifically focus on text-conditioned latent diffusion models \cite{rombach2022ldm}.

Let $\mathrm{PromptSplit}(x_{1:n}, x_{1:n}\vert t_{1:n})$ denote the $\| \widehat{\Lambda}_{X,Y|T} \|_F^2$ in the latent space $\mathcal{Z}$ and $\mathrm{Joint\text{-}Diversity}(x_{1:n}, t_{1:n})$ denote $\| C_{X\otimes T} \|_F^2$ which promotes diversity between samples with correlated prompts. At  step $\tau$, we augment the classifier-free update \citep{Ho2022ClassifierFreeDG} with an ascent step where $\eta_\tau>0$ is the guidance scale at iteration $\tau$, and $\rho$ is the guidance scale of the diversity term:
\begin{align}\label{eq:sp-guidance}
\boldsymbol{z}^{(n)}_{\tau-1}
\;\leftarrow&\;
\mathrm{Sampler}\bigl(\boldsymbol{z}^{(n)}_\tau,\,
\hat{\epsilon}_\theta(\boldsymbol{z}^{(n)}_\tau,\tau,t_n)\bigr)\\
\; &-\eta_\tau\,\Bigl(\nabla_{\boldsymbol{z}^{(n)}} \mathrm{PromptSplit}(x_{1:n}, y_{1:n}\vert t_{1:n}),\nonumber\\
&+\rho\,\nabla_{\boldsymbol{z}^{(n)}} \mathrm{Joint\text{-}Diversity}(x_{1:n}, t_{1:n})\Bigl)\nonumber
\end{align}


%% file: sec/4_numerical.tex
\begin{figure*}[t]
    \centering
    \includegraphics[width=\linewidth]{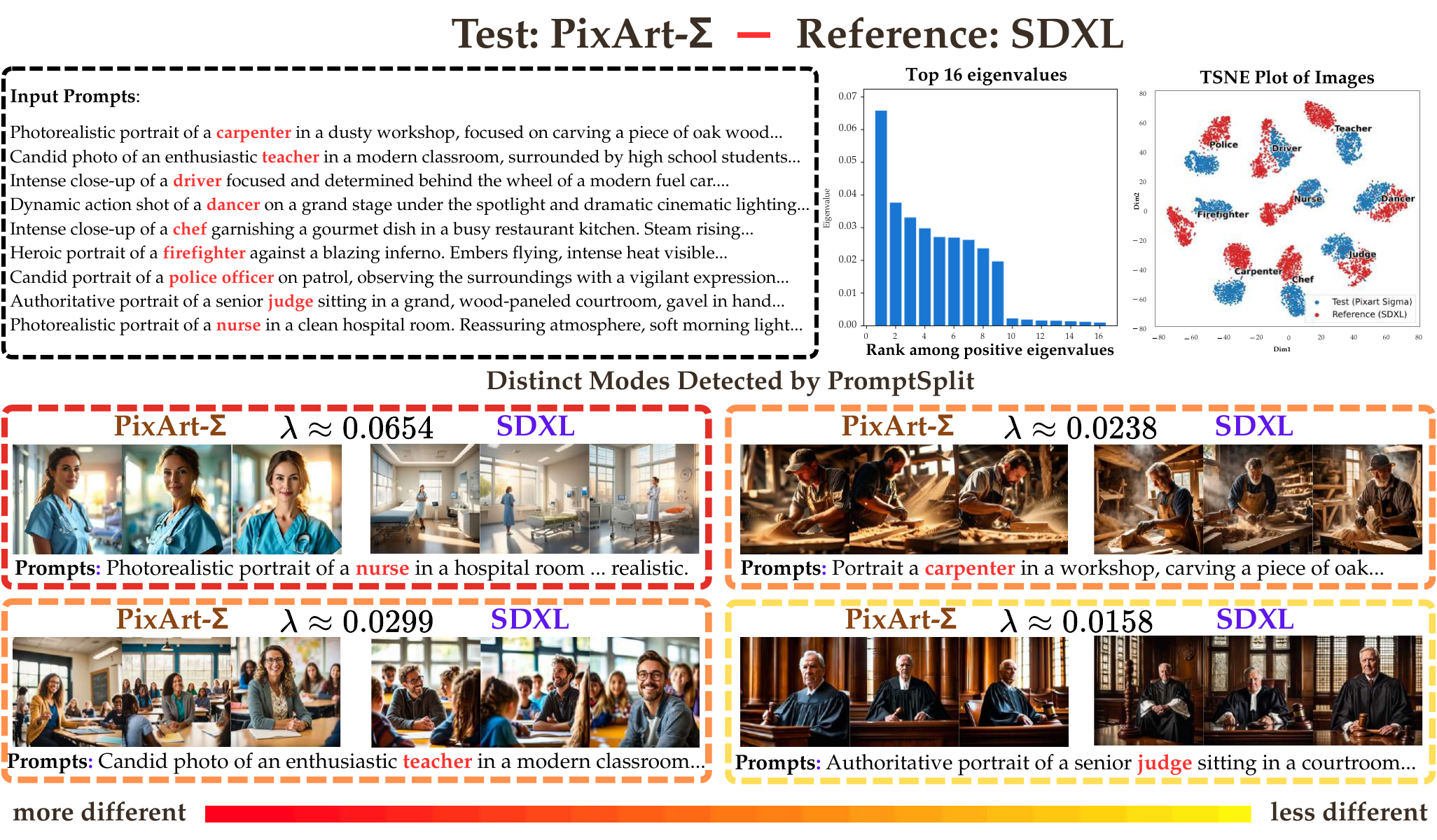}
    \caption{PromptSplit uncovers occupation-based divergences between PixArt-$\Sigma$ (test model) and SDXL (reference model). (Top middle) Largest eigenvalues barplot highlighting top nine distinct clusters. (Top right) t-SNE projection of images embeddings, revealing occupational clusters. (Bottom) Representative generated images for different $\lambda$ values.}
    \label{fig:sdxl_pixart}
\end{figure*}

\begin{figure*}[h]
    \centering
    \includegraphics[width=\linewidth]{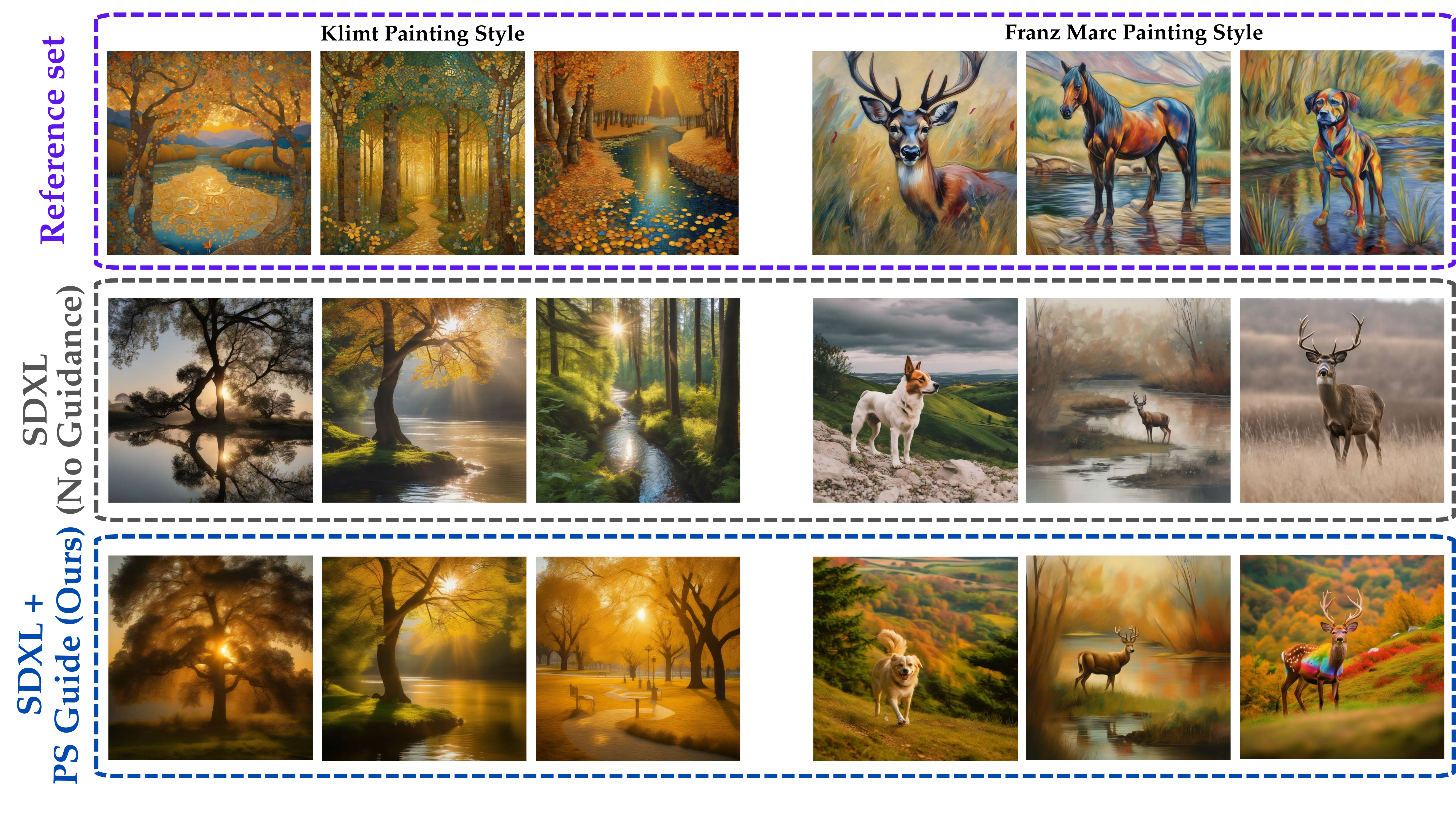}\vspace{-3mm}
    \caption{Qualitative comparison of reference set and PS-guided image generation with SDXL.}
    \label{fig:guidance_sdxl}
\end{figure*}


We evaluated PromptSplit on text guided image generation and text-to-text generative models (LLMs). In the following we provide experimental settings and comparison of kernel based PromptSplit vs. scalable random projection version of PromptSplit showing that the method was able to distinguish the disagreement in generative models.

\textbf{Models.}
In our experiments on text-to-image models we used state-of-the-art generative models, Stable Diffusion XL \cite{sdxl}, PixArt-$\Sigma$ \cite{chen2024pixartsigmaweaktostrongtrainingdiffusion}, and Kandinsky \cite{arkhipkin2024kandinsky3}. For text-to-text models we used Llama~3.2 \cite{llama}, Gemma~3 \cite{gemma3}, DeepSeek-r1 \cite{deepseek}, and Qwen~3 \cite{qwen3}.

\textbf{Datasets.} We used standard image datasets including MNIST \cite{mnist}, MNIST-M \cite{mnistm} (colored background MNIST). For text-to-image scenarios we use MS-COCO validation set \cite{mscoco}, and for text-to-text models we used NQ-Open \cite{nq_open} (Open-domain QA from Natural Questions).

\textbf{Experimental Settings.} We approximated the Gaussian kernel using r = 3000 random Fourier features to compute the PromptSplit differential kernel covariance matrix. The bandwidth $\sigma$ was determined according to the method in \cite{ospanov2024fkea}, selecting embedding bandwidths such that the gap between the largest eigenvalues was less than 0.01. We used DinoV2-giant \cite{oquab2023dinov2} to embed images and Sentence-Bert \cite{sbert} for texts. Full details of the PromptSplit and Random Projection PromptSplit algorithms appear in Algorithms \ref{alg:promptsplit} and \ref{alg:rp-promptsplit}. Experiments were run on four RTX A5000 and two RTX 4090 GPUs.


\subsection{Validation of PromptSplit in settings with known groundtruth}

\textbf{MNIST-M Color Disagreement.} We used 20 prompts and sampled 1000 images per prompt for colored and grayscale digits in MNIST. To control the setting differences, we sampled grayscale digits for colored prompts 5 to 9 in test dataset and colored digits for grayscale prompt 0 to 4 in reference dataset as shown in figure \ref{fig:mnist}. We used random projection version of PromptSplit due to the large size of datasets which successfully identified the top ten prompts with distinct images.


In the Appendix, we further present additional numerical results validating the outputs of PromptSplit in scenarios where we know the ground truth for text-to-image models (figure \ref{fig:sdxl_controlled}) and LLMs (figures \ref{fig:llama_gemma}, \ref{fig:deepseek}). For both text-to-image models and LLM, we studied wether the model can successfully capture the disagreement and ignore the similarities. We also validate the model in synthetic setting (figure \ref{fig:synthetic}).

\subsection{Application of PromptSplit to T2I Models}

\textbf{Occupation Level Disagreement Identification.}
We generated 500 images for nine different occupation prompt clusters with SDXL and PixArt-$\Sigma$. PromptSplit identified clusters with most disagreements (Nurse), some biases in age or gender (Carpenter and Teacher), and minor differences (Judge) due to the magnitude of the resulting eigenvalue. the results are further supported by the t-SNE scatter plot of images
embeddings.

We refer the application of PromptSplit for different models on MS-COCO captions to appendix (figures \ref{fig:pixart_sdxl_mscoco}, \ref{fig:kandinsky_mscoco_mscoco}, \ref{fig:sdxl_mscoco_mscoco}, \ref{fig:mscoco_sdxl_mscoco}, \ref{fig:pixart_mscoco_mscoco}, \ref{fig:mscoco_pixart_mscoco}).

\subsection{Application of PromptSplit to LLMs}

\textbf{NQ-Open Experiment.} We generated answers for 20000 NQ-Open validation questions with Qwen~3 as test and Gemma~3  as Reference \ref{fig:introduction}. We applied PromptSplit to identify top cluster prompts (questions about actors and US. presidents) with different generated answers.

We refer to appendix for more comparison between other LLMs pairs and sampled similar cluster for sanity check (figures \ref{fig:gemma_llama_nq}, \ref{fig:llama_nq_nq}, \ref{fig:nq_llama_nq}).

\subsection{PromptSplit Guidance for Distribution Matching in LDMs.}
To further study one of the applications of PromptSplit, we use PromptSplit to guide the diffusion process in the latent space of the latent diffusion models as formulated in Section~\ref{sec: maintext guidance}. We used 100 reference samples for each of the painting styles and guided the Stable Diffusion-XL to align the distribution of generated samples to the reference set. As shown in Figure~\ref{fig:guidance_sdxl}, adding PS guidance improves the alignment of SDXL with the reference dataset. For the additional results, we refer to the Appendix.

\subsection{Ablation Studies}
We refer the reader to the Appendix (figures \ref{fig:r_ablation1}, \ref{fig:r_ablation2}, \ref{fig:n_ablation1}, \ref{fig:n_ablation2}) for additional ablation studies examining the effect of (i) the number of samples and (ii) the number of random Fourier features used to approximate the Gaussian kernel.


%% file: sec/5_conclusion.tex
\begin{table}[t]
\centering
\caption{Runtime comparison (in seconds) of different PromptSplit algorithms on the MS-COCO val 2014 dataset.}
\label{tab:mscoco-comparison-time}
\small  
\begingroup
\setlength{\tabcolsep}{25pt}  
\begin{tabular}{lcccc}  
\toprule
\multirow{2}{*}{Sample Size} & \multirow{2}{*}{Kernel Method} & \multicolumn{3}{c}{Random Projection} \\
\cmidrule(lr){3-5}
 & & (r = 1000) & (r = 2000) & (r = 3000) \\
\midrule
n = 1000 & 0.856 & 4.186 & 43.674 & 88.156 \\
n = 2000 & 6.322 & 4.660 & 43.995 & 147.690 \\
n = 5000 & 111.890 & 5.391 & 40.290 & 162.831 \\
n = 10000 & 902.818 & 7.830 & 64.162 & 158.452 \\
n = 20000 & --- & 12.033 & 77.393 & 188.620 \\
n = 30000 & --- & 17.765 & 72.271 & 206.996 \\
\bottomrule
\end{tabular}
\endgroup
\end{table}

In this work, we introduced PromptSplit, a prompt-aware spectral framework for comparing prompt-guided generative models through the eigenspectrum of joint prompt–output kernel covariance differences. By coupling prompt and output representations via tensor-product embeddings, PromptSplit enables a structured analysis of where and how model behaviors diverge as a function of the input prompt, moving beyond aggregate quality metrics toward prompt-conditioned comparison. We proposed a kernel formulation that admits an efficient spectral implementation, provided theoretical guarantees for a scalable random-projection approximation, and demonstrated the effectiveness of the approach on both synthetic settings and real text-to-image and text-to-text generative models.

PromptSplit is designed as an embedding-based, second-order spectral method, and its analysis reflects the information captured by the chosen prompt and output representations. While this design enables scalability and interpretability, it does not aim to characterize all higher-order aspects of conditional generative behavior. The random-projection scheme introduces a standard accuracy–efficiency trade-off that is controlled by the projection dimension and can be tuned in practice. The current formulation focuses on pairwise model comparison under matched prompt distributions; extending the framework to multi-model comparisons, adaptive projection strategies, or evolving model families represents natural directions for future work.

%% file: sec/6_appendix_proofs.tex
\subsection{Proof of Proposition~\ref{prop:hadamard}}
We start by writing the empirical covariance--difference operator explicitly.
By definition,
\[
\widehat{\Lambda}_{X,Y\mid T}
=
\frac{1}{n}\sum_{i=1}^n 
\phi_\otimes(t_i,x_i)\phi_\otimes(t_i,x_i)^\top
-
\frac{\eta}{m}\sum_{j=1}^m
\phi_\otimes(t'_j,y_j)\phi_\otimes(t'_j,y_j)^\top .
\]

Consider the following stacked feature matrices:
\[
\Psi_X
=
\begin{bmatrix}
\phi_\otimes(t_1,x_1)^\top\\
\vdots\\
\phi_\otimes(t_n,x_n)^\top
\end{bmatrix}
\in \mathbb{R}^{n\times D},
\qquad
\Psi_Y
=
\begin{bmatrix}
\phi_\otimes(t'_1,y_1)^\top\\
\vdots\\
\phi_\otimes(t'_m,y_m)^\top
\end{bmatrix}
\in \mathbb{R}^{m\times D},
\]
where $D=d_T d_X$.
Then, we can write
\[
\widehat{\Lambda}_{X,Y\mid T}
=
\Psi^\top
\begin{bmatrix}
\frac{1}{n}I_n & 0\\
0 & -\frac{\eta}{m}I_m
\end{bmatrix}
\Psi,
\qquad
\Psi
=
\begin{bmatrix}
\Psi_X\\
\Psi_Y
\end{bmatrix}.
\]

Next, consider the corresponding kernel matrix
\[
K_{X,\eta Y\mid T}
=
\begin{bmatrix}
\frac{1}{n}K_{TT}\odot K_{XX}
&
\frac{1}{\sqrt{nm}}K_{TT'}\odot K_{XY}
\\[1mm]
-\frac{\eta}{\sqrt{nm}}K_{TT'}^\top\odot K_{XY}^\top
&
-\frac{\eta}{m}K_{T'T'}\odot K_{YY}
\end{bmatrix}.
\]

Using the tensor-product identity $
\langle \phi_T(t)\otimes\phi_X(x),\,
        \phi_T(t')\otimes\phi_X(x')\rangle
=
k_T(t,t')\,k_X(x,x')$, we observe that
\[
K_{X,\eta Y\mid T}
=
\begin{bmatrix}
\Psi_X\\
\Psi_Y
\end{bmatrix}
\begin{bmatrix}
\frac{1}{n}I_n & 0\\
0 & -\frac{\eta}{m}I_m
\end{bmatrix}
\begin{bmatrix}
\Psi_X\\
\Psi_Y
\end{bmatrix}^{\!\top}
=
\Psi
\begin{bmatrix}
\frac{1}{n}I_n & 0\\
0 & -\frac{\eta}{m}I_m
\end{bmatrix}
\Psi^\top .
\]

Hence, the following equations hold
\[
\widehat{\Lambda}_{X,Y\mid T}
=
\Psi^\top D \Psi,
\qquad
K_{X,\eta Y\mid T}
=
\Psi D \Psi^\top ,
\quad
D=
\mathrm{diag}\!\left(\tfrac{1}{n}I_n,-\tfrac{\eta}{m}I_m\right).
\]

It is a standard linear-algebra fact that the matrices
$\Psi^\top D \Psi$ and $\Psi D \Psi^\top$ have the same nonzero eigenvalues (due to flipped multiplication order).
Therefore, $\widehat{\Lambda}_{X,Y\mid T}$ and $K_{X,\eta Y\mid T}$
share the same nonzero spectrum. Finally, let $u=[u_{1:n};u_{(n+1):(n+m)}]$ be an eigenvector of
$K_{X,\eta Y\mid T}$ with eigenvalue $\lambda\neq 0$.
Define vector $v$ as follows:
\[
v
=
\Psi^\top u
=
\sum_{i=1}^n u_i\,\phi_\otimes(t_i,x_i)
+
\sum_{j=1}^m u_{n+j}\,\phi_\otimes(t'_j,y_j).
\]
Then, we have the following
\[
\widehat{\Lambda}_{X,Y\mid T}v
=
\Psi^\top D \Psi \Psi^\top u
=
\Psi^\top D (K_{X,\eta Y\mid T}u)
=
\lambda\,\Psi^\top u
=
\lambda v,
\]
which proves that $v$ is an eigenvector of $\widehat{\Lambda}_{X,Y\mid T}$
associated with the same eigenvalue $\lambda$. This completes the proof.

\subsection{Proof of Theorem~\ref{thm:rp}}

First, we show the unbiasedness and boundedness of the random-feature product kernel. By Bochner's theorem, for shift-invariant $k_T(t,t')=\kappa_T(t-t')$ and
$k_X(x,x')=\kappa_X(x-x')$, there exist spectral measures $\mu_T,\mu_X$ such that
\[
k_T(t,t')=\mathbb{E}_{\omega_t\sim\mu_T}\big[\cos(\omega_t^\top(t-t'))\big],
\qquad
k_X(x,x')=\mathbb{E}_{\omega_x\sim\mu_X}\big[\cos(\omega_x^\top(x-x'))\big].
\]
Using independence between $\omega_{t,\ell}$ and $\omega_{x,\ell}$, and the identity
$\cos(a)\cos(b)=\tfrac{1}{2}\cos(a+b)+\tfrac{1}{2}\cos(a-b)$, we may equivalently use the
single cosine feature  to represent the product kernel:
\begin{align}
\mathbb{E}\Big[g_\ell\bigl((t,x),(t',x')\bigr)\Big]
&=
\mathbb{E}_{\omega_t,\omega_x}\Big[\cos\big(\omega_t^\top(t-t')+\omega_x^\top(x-x')\big)\Big]
\nonumber\\
&=
\mathbb{E}_{\omega_t}\big[\cos(\omega_t^\top(t-t'))\big]\;
\mathbb{E}_{\omega_x}\big[\cos(\omega_x^\top(x-x'))\big]
\label{eq:unbiased-prod}\\
&=
k_T(t,t')\,k_X(x,x').
\nonumber
\end{align}
Moreover, for every $\ell$ and every pair of inputs,
\begin{equation}
\label{eq:bounded-g}
\bigl|g_\ell\bigl((t,x),(t',x')\bigr)\bigr|\le 1,
\qquad
\bigl|\widetilde{k}_\otimes\bigl((t,x),(t',x')\bigr)\bigr|\le 1.
\end{equation}

Next, we decompose the matrix as an average of i.i.d.\ bounded matrices. To do this, define $N:=n+m$, and for every $\ell\in[r]$, consider the \emph{single-feature} block matrix
$K^{(\ell)}_{X,\eta Y|T}\in\mathbb{R}^{N\times N}$ by
\begin{equation}
\label{eq:Kell}
K^{(\ell)}_{X,\eta Y|T}
=\hspace{-1.5mm}
\begin{bmatrix}
\frac{1}{n}G^{(\ell)}_{XX}
& \frac{1}{\sqrt{nm}}G^{(\ell)}_{XY}\vspace{2mm}\\
-\frac{\eta}{\sqrt{nm}}(G^{(\ell)}_{XY})^\top
& -\frac{\eta}{m}G^{(\ell)}_{YY}
\end{bmatrix},
\end{equation}
where the entries are
\begin{align*}
\bigl(G^{(\ell)}_{XX}\bigr)_{ij}
&:=g_\ell\bigl((t_i,x_i),(t_j,x_j)\bigr),
\qquad i,j\in[n],
\\
\bigl(G^{(\ell)}_{YY}\bigr)_{jj'}
&:=g_\ell\bigl((t'_j,y_j),(t'_{j'},y_{j'})\bigr),
\qquad j,j'\in[m],
\\
\bigl(G^{(\ell)}_{XY}\bigr)_{ij}
&:=g_\ell\bigl((t_i,x_i),(t'_j,y_j)\bigr),
\qquad i\in[n],\ j\in[m].
\end{align*}
By construction of $\widetilde{k}_\otimes$, the proxy matrix is the
average of these blocks:
\begin{equation}
\label{eq:avg}
\widetilde{K_r}_{X,\eta Y|T}
=\frac{1}{r}\sum_{\ell=1}^r K^{(\ell)}_{X,\eta Y|T}.
\end{equation}
Taking expectation and using~\eqref{eq:unbiased-prod} entrywise shows that
\begin{equation}
\label{eq:exp}
\mathbb{E}\Big[K^{(\ell)}_{X,\eta Y|T}\Big]
=
K_{X,\eta Y|T}.
\end{equation}

From~\eqref{eq:bounded-g}, every entry of $G^{(\ell)}_{XX},G^{(\ell)}_{YY},G^{(\ell)}_{XY}$
has magnitude at most $1$. Hence,
\begin{align}
\left\|\frac{1}{n}G^{(\ell)}_{XX}\right\|_F^2
&\le
\sum_{i=1}^n\sum_{j=1}^n \left(\frac{1}{n}\right)^2
=1,
\label{eq:frob-xx}\\
\left\|\frac{1}{\sqrt{nm}}G^{(\ell)}_{XY}\right\|_F^2
&\le
\sum_{i=1}^n\sum_{j=1}^m \left(\frac{1}{\sqrt{nm}}\right)^2
=1,
\label{eq:frob-xy}\\
\left\|\frac{\eta}{m}G^{(\ell)}_{YY}\right\|_F^2
&\le
\sum_{j=1}^m\sum_{j'=1}^m \left(\frac{\eta}{m}\right)^2
=\eta^2.
\label{eq:frob-yy}
\end{align}
Using the matrix formulation, we obtain the following
\begin{align}
\bigl\|K^{(\ell)}_{X,\eta Y|T}\bigr\|_F^2
&\le
\left\|\frac{1}{n}G^{(\ell)}_{XX}\right\|_F^2
+
\left\|\frac{1}{\sqrt{nm}}G^{(\ell)}_{XY}\right\|_F^2
+
\left\|\frac{\eta}{\sqrt{nm}}(G^{(\ell)}_{XY})^\top\right\|_F^2
+
\left\|\frac{\eta}{m}G^{(\ell)}_{YY}\right\|_F^2
\nonumber\\
&\le
1+1+\eta^2+\eta^2
=
2(1+\eta^2)
\label{eq:frob-Kell}
\end{align}
Therefore,
\begin{equation}
\label{eq:frob-Kell-simple}
\bigl\|K^{(\ell)}_{X,\eta Y|T}\bigr\|_F \le \sqrt{2+2\eta^2}
\end{equation}
The same argument applied to the expectation in~\eqref{eq:exp} yields
$\|K_{X,\eta Y|T}\|_F\le \sqrt{2+2\eta^2}$, and thus, by the triangle inequality,
\begin{equation}
\label{eq:centered-bound}
\left\|K^{(\ell)}_{X,\eta Y|T}-K_{X,\eta Y|T}\right\|_F
\le
\bigl\|K^{(\ell)}_{X,\eta Y|T}\bigr\|_F+\bigl\|K_{X,\eta Y|T}\bigr\|_F
\le \sqrt{8+8\eta^2}
\end{equation}

Now, we consider the Hilbert space $(\mathbb{R}^{N\times N},\langle\cdot,\cdot\rangle_F)$. Consider the centered random matrices
\[
X_\ell:=K^{(\ell)}_{X,\eta Y|T}-K_{X,\eta Y|T},
\qquad \ell=1,\ldots,r.
\]
Then, $X_1,\ldots,X_r$ are i.i.d. random vectors satisfying $\mathbb{E}[X_\ell]=0$, and by~\eqref{eq:centered-bound},
$\|X_\ell\|_F\le 4$ almost surely.
A Hoeffding inequality for Hilbert-space-valued random vectors
(\citep[Lemma~11]{sutherland2018efficient}) implies that, for every $\delta\in(0,1)$,
with probability at least $1-\delta$,
\begin{equation}
\label{eq:hoeff}
\left\|\frac{1}{r}\sum_{\ell=1}^r X_\ell\right\|_F
\le
\frac{4}{\sqrt{r}}
\left(1+\sqrt{2\log\frac{1}{\delta}}\right).
\end{equation}
Since $\frac{1}{r}\sum_{\ell=1}^r X_\ell
=\widetilde{K_r}_{X,\eta Y|T}-K_{X,\eta Y|T}$ by~\eqref{eq:avg} and $\mathbb{E}[K^{(\ell)}]=K$,
we obtain
\begin{equation}
\label{eq:frob-dev}
\bigl\|\widetilde{K_r}_{X,\eta Y|T}-K_{X,\eta Y|T}\bigr\|_F
\le
\sqrt{\frac{8+8\eta^2}{r}}
\left(1+\sqrt{2\log\tfrac{1}{\delta}}\right).
\end{equation}

Since both $K_{X,\eta Y|T}$ and $\widetilde{K_r}_{X,\eta Y|T}$ are symmetric, the
Hoffman--Wielandt inequality gives
\begin{equation}
\label{eq:hw}
\bigl\|\boldsymbol{\lambda}(K_{X,\eta Y|T})
-\boldsymbol{\lambda}(\widetilde{K_r}_{X,\eta Y|T})\bigr\|_2
\le
\bigl\|K_{X,\eta Y|T}-\widetilde{K_r}_{X,\eta Y|T}\bigr\|_F.
\end{equation}
Combining these two inequalities proves the stated inequality.

%% file: sec/6_appendix_results.tex
\subsection{Image Style Disagreement.}
In figure \ref{fig:introduction} we generated three clusters of prompts and sampled 1000 images per cluster with different styles. For test dataset we used oil painting style for cityscape images, and pop art style for forests which differs from photo realistic style images generated for reference. We also generated photo realistic mountain images for both test and reference datasets. PromptSplit kernel-based algorithm successfully identified the two distinct modes in style with significant larger eigenvalues.

\begin{figure*}
    \centering
    \includegraphics[width=\linewidth]{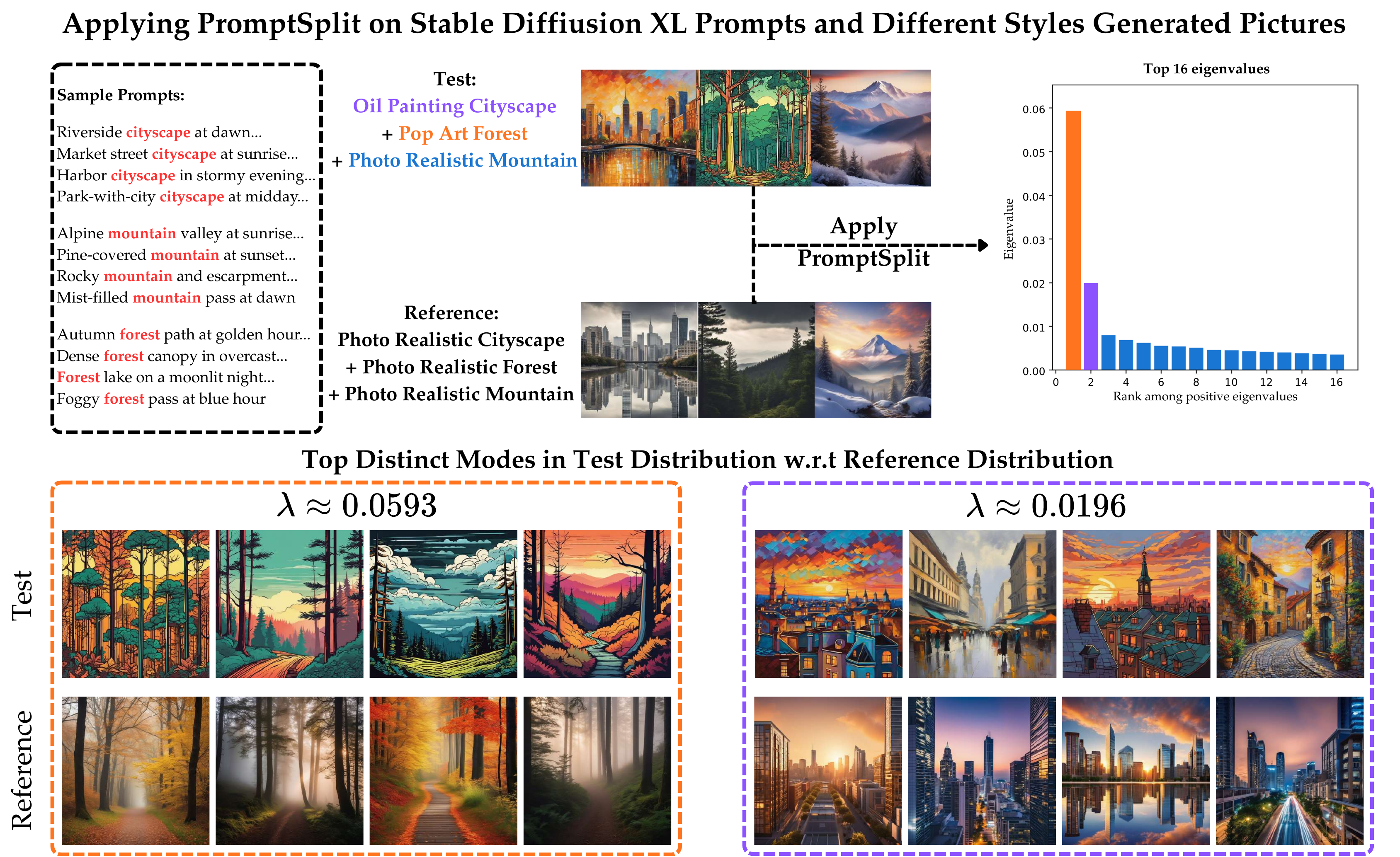}\vspace{-2mm}
    \caption{PromptSplit detected style and scene disagreements in a controlled text-to-image setting. (Top) Sample prompts, generated outputs, and bar plot of top 16 eigenvalues. (Bottom) Strongest samples for top identified distinct modes.}
    \label{fig:sdxl_controlled}
\end{figure*}

\subsection{Answer Distribution Disagreement Detection}
We compare Llama 3.2 (Test) and Gemma 3 (Reference) on three prompt clusters (Figure~\ref{fig:llama_gemma} for celebrity, American city, and fast modern fuel car, using 100 prompts per cluster and 10,000 generated answers per cluster (30,000 prompt–answer pairs per model). The per-cluster answer dispersions show a clear difference in answer generation: Gemma collapses strongly onto one or two dominant entities in all clusters, while Llama’s answers are more distributed across multiple names. Applying PromptSplit to the aggregated prompt–answer pairs yields a small number of dominant test-dominant disagreement directions, and the leading samples for these modes isolate interpretable, cluster-specific differences (top three modes respectively shows entities in celebrities, modern cars, and cities that are most different from the reference set answers) which aligns with variations in prompt/answer behaviors visible in the dispersions.

\begin{figure*}[t]
    \centering
    \includegraphics[width=\linewidth]{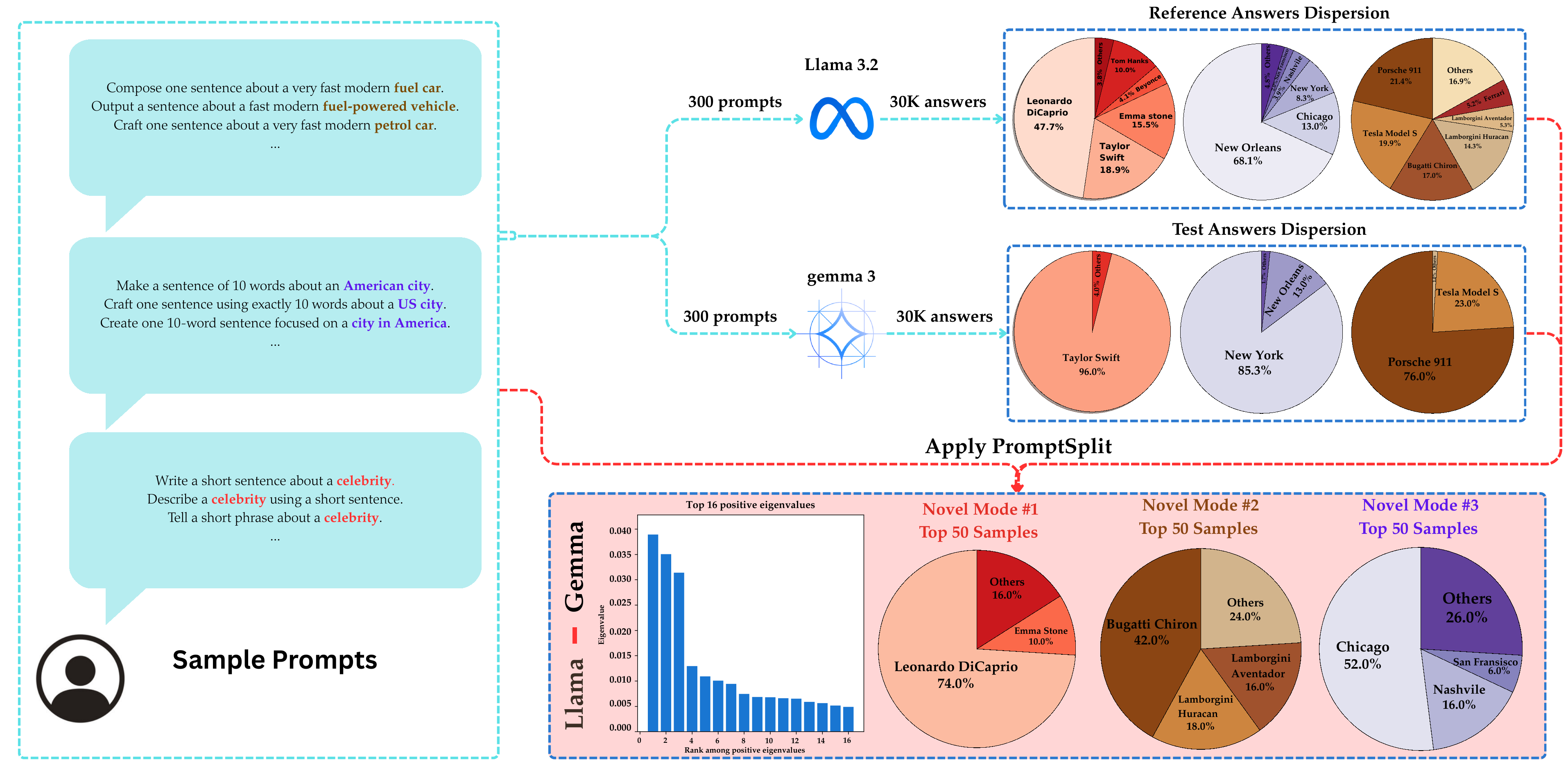}
    \caption{PromptSplit reveals prompt-dependent answer-mode differences between Llama 3.2 (test) and Gemma 3 (reference) on three prompt clusters (celebrity, American city, and fast modern fuel car). Top: per-cluster empirical answer dispersion for each model, highlighting differences in concentration vs. diversity. Bottom: PromptSplit identified three clusters with leading positive eigenvalues (left) and top answers only generated by test model.}
    \label{fig:llama_gemma}
\end{figure*}

\subsection{Controlled Answers Disagreement.}
To validate that PromptSplit can recover known differences, we construct a controlled text-to-text setting (Figure~\ref{fig:deepseek}) where both “reference” and “test” outputs are produced by the same DeepSeek-r1 \cite{deepseek} model, but with different answer-control constraints. Prompts are partitioned into three clusters, and the control is intentionally applied so that the test condition differs from the reference primarily in the movie-genre and historical-phenomena clusters, while the third cluster is a control where reference and test are intended to match. We apply PromptSplit by embedding prompts and answers, constructing a joint prompt–answer similarity structure, and computing the leading eigen-directions of the covariance-difference operator, where large positive eigenvalues indicate the most prominent test-dominant disagreement modes. The resulting eigenspectrum shows two dominant positive modes, consistent with the two clusters where answers were deliberately shifted, and inspecting the leading samples for these modes demonstrates human-interpretable differences (Horror w.r.t Action and Poleponnesian War w.r.t The Fall of Roman Empire) that align with the intended controlled shifts.

\begin{figure*}[t]
    \centering
    \includegraphics[width=\linewidth]{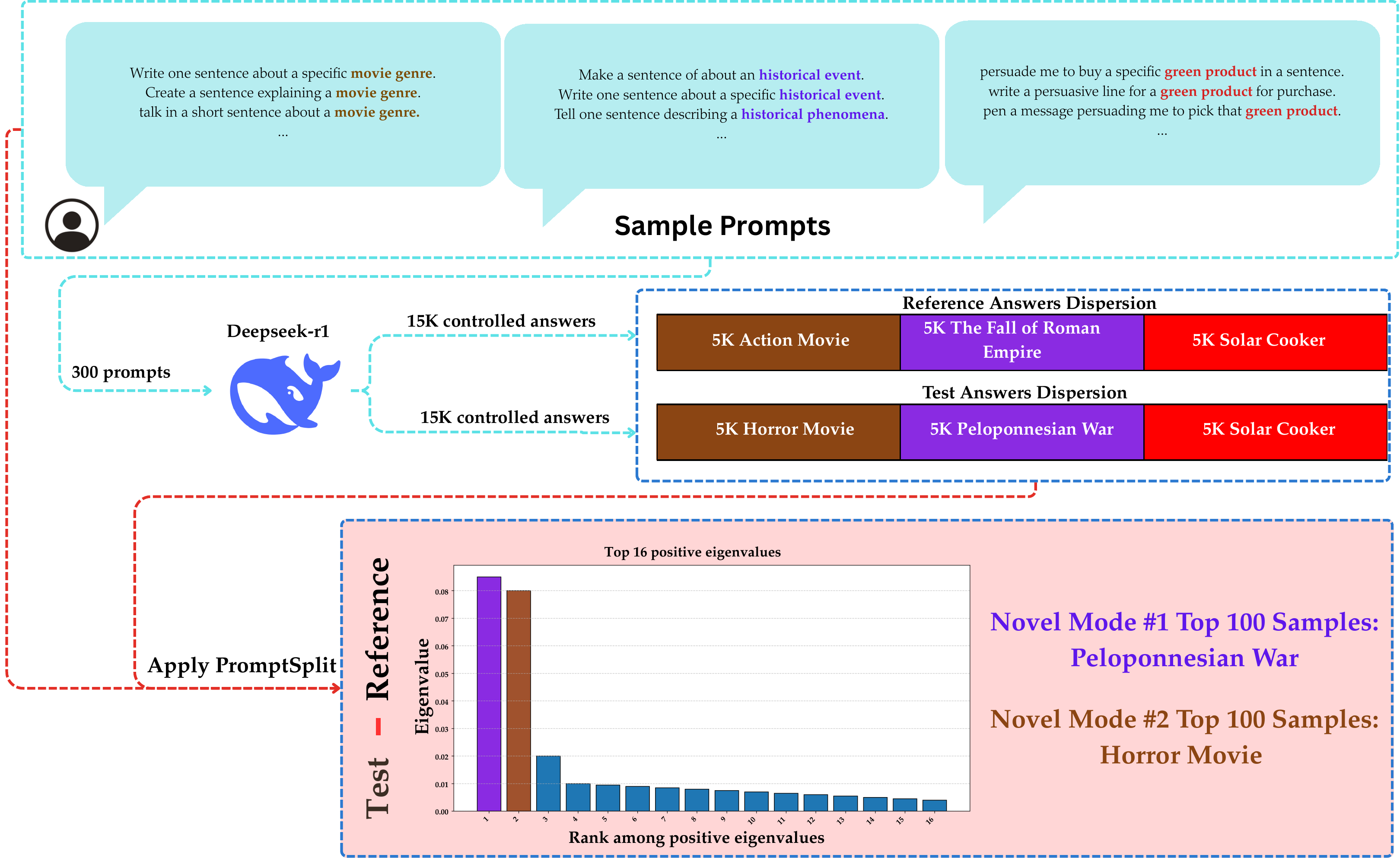}
    \caption{Top: three prompt clusters (movie genre, historical event/phenomena, and green-product persuasion) Middle: empirical answer dispersion for reference vs. test, illustrating that the control mechanism induces systematic shifts in two clusters while keeping the third comparatively stable. Bottom: PromptSplit’s leading positive eigenvalues identify the strongest test-dominant disagreement directions, and highest-attribution samples shows two interpretable novel modes (Peloponnesian War and Horror Movie).
}
    \label{fig:deepseek}
\end{figure*}

\subsection{Ablation Study on Number of random Fourier features}
We evaluate projection dimension $r \in \{800, 1500, 2000, 3000, 5000\}$ on the full 30k-pair SDXL vs. PixArt-$\Sigma$ comparison ($\eta$ = 1, DINOv2-giant + Sentence-BERT embeddings, bandwidth selected via eigenvalue-gap heuristic) (Figures \ref{fig:r_ablation1}, \ref{fig:r_ablation2}).
At r = 800, primary modes are recovered but eigenvalues are compressed and smaller modes noisier. By r = 1500–2000, eigenvalue magnitudes increase, mode ranking stabilizes, and prompt attribution sharpens. At r = 3000 (default), top modes match those at r = 5000 almost identically in spectrum shape, prompt semantics, and visual patterns.
These results align with Theorem \ref{thm:rp}  $O(1/\sqrt{r})$ eigenvalue-vector bound and confirm that modest r suffices for approximation of leading disagreement directions at scale.

\subsection{Ablation Study on Number of Samples}
To investigate the effect of dataset size on the quality and stability of the detected disagreement modes, we evaluate the random-projection variant of PromptSplit across varying numbers of prompt–output pairs (Figures \ref{fig:n_ablation1}, \ref{fig:n_ablation2}). We use the MS-COCO 2014 validation set (30k captions) as the base and create subsampled versions with $n = m \in \{5k, 10k, 20k, 30k, 50k, 90k\}$ pairs per model/dataset. All experiments employ the same hyperparameters (r = 3000, $\eta$ = 1, DINOv2-giant + Sentence-BERT embeddings, bandwidth selected via eigenvalue-gap heuristic).
Figure X (top row) shows the top-10 eigenvalues and the strongest-attributed prompts/images for the pairwise comparison SDXL vs. PixArt-$  \Sigma  $ across different sample sizes. We observe that:

\begin{itemize}
    \item With as few as 5k–10k samples, the method already identifies the dominant disagreement modes (e.g., tennis-court geometry, plate placement, ramp/sideboard composition), with the largest eigenvalues corresponding to semantically consistent prompt clusters, however there are minor flaws in some modes.
    \item At 30k samples (the full MS-COCO val2014 size), eigenvalue magnitudes stabilize, and the ranking of top modes becomes highly consistent with the full-dataset run.
    Beyond 30k (up to 90k subsampled with replacement from the same caption pool), further gains are marginal: the top 3–5 modes remain nearly identical in attributed prompts and visual patterns, while smaller modes show minor fluctuations due to sampling variance.
    \item Beyond 30k (up to 90k generated with different seeds from the same 30K caption), further gains are marginal: the top 3–5 modes remain nearly identical in attributed prompts and visual patterns, while smaller modes show minor fluctuations due to added samples disagreements.
\end{itemize}

These results indicate that PromptSplit is robust even at moderate sample sizes and 30k pairs are typically sufficient to reliably recover the principal directions of prompt-dependent disagreement in real-world text-to-image settings. Importantly, the kernel-based (non-projected) formulation becomes computationally prohibitive beyond ~10K-15K samples $(O((n+m)^3)$ scaling), showing the necessity of the random-projection approximation for large-scale analysis.
Runtime measurements (single A5000 GPU) are reported in Table \ref{tab:mscoco-comparison-time}.

\subsection{Synthetic Gaussian Mixture.}
In a controlled synthetic study (Figure~\ref{fig:synthetic}), eight 50-dimensional Gaussian components simulate test embeddings and eight simulate reference embeddings, with 8-dimensional one-hot prompts shared within each cluster and 100 samples per component; joint projections visibly separate clusters and the eigen-spectrum of the covariance-difference operator behaves predictably: when the test and reference mixtures differ in k components, PromptSplit yields k positive and k negative principal eigenvalues, while when all components coincide, the spectrum collapses near zero—empirically validating that the method detects precisely the number and directionality of cluster-level disagreements.

\subsection{Additional NQ-Open Experiments}
Figures \ref{fig:gemma_llama_nq}, \ref{fig:llama_nq_nq}, \ref{fig:nq_llama_nq} extend the NQ-Open study to additional LLM pairs and show that PromptSplit isolates prompt-conditioned disagreement in a consistent, interpretable way across comparisons. In each figure, we report the leading positive eigenmodes and we qualitatively inspect the highest-attribution question clusters to reveal the prompt families that most strongly separate the two models’ behaviors. Across pairs, the dominant modes typically correspond to coherent topical categories (e.g., particular entity-centric question types), where the models diverge in answer selection, factual framing, or phrasing. Importantly, alongside these high-disagreement directions, we also include a representative low-eigenvalue (or “similar”) mode which shows questions that receive weak attribution under PromptSplit and yield largely aligned answers across the two models.

\subsection{Disagreement Identification on MS-COCO Captions.}
We generated images using three state-of-the-art text-to-image models—Stable Diffusion XL (SDXL), PixArt-$ \Sigma $, and Kandinsky on the 30,000 captions from the MS-COCO 2014 validation set. Due to the large dataset size, we applied the scalable random-projection variant of PromptSplit (with projection dimension $  r=3000  $) to identify prompt clusters that induce systematic behavioral disagreements in the generated images (Figures \ref{fig:pixart_sdxl_mscoco}, \ref{fig:kandinsky_mscoco_mscoco}, \ref{fig:sdxl_mscoco_mscoco}, \ref{fig:mscoco_sdxl_mscoco}, \ref{fig:pixart_mscoco_mscoco}, \ref{fig:mscoco_pixart_mscoco})

We performed pairwise comparisons between each pair of models (SDXL vs. PixArt-$ \Sigma $, SDXL vs. Kandinsky, PixArt-$  \Sigma  $ vs. Kandinsky) as well as comparisons of each model against the real MS-COCO validation images (serving as a reference distribution). For each comparison, the top 5 identified modes of disagreement with the largest positive eigenvalues. Each mode is illustrated a selection of generated (or real reference) images.

These visualizations highlight interpretable prompt families where the models diverge in visual style, composition, object placement, realism, or alignment with the input text, complementing aggregate metrics such as FID or CLIPScore that do not capture prompt-conditioned differences.

\begin{figure*}
    \centering
    \includegraphics[width=0.95\linewidth]{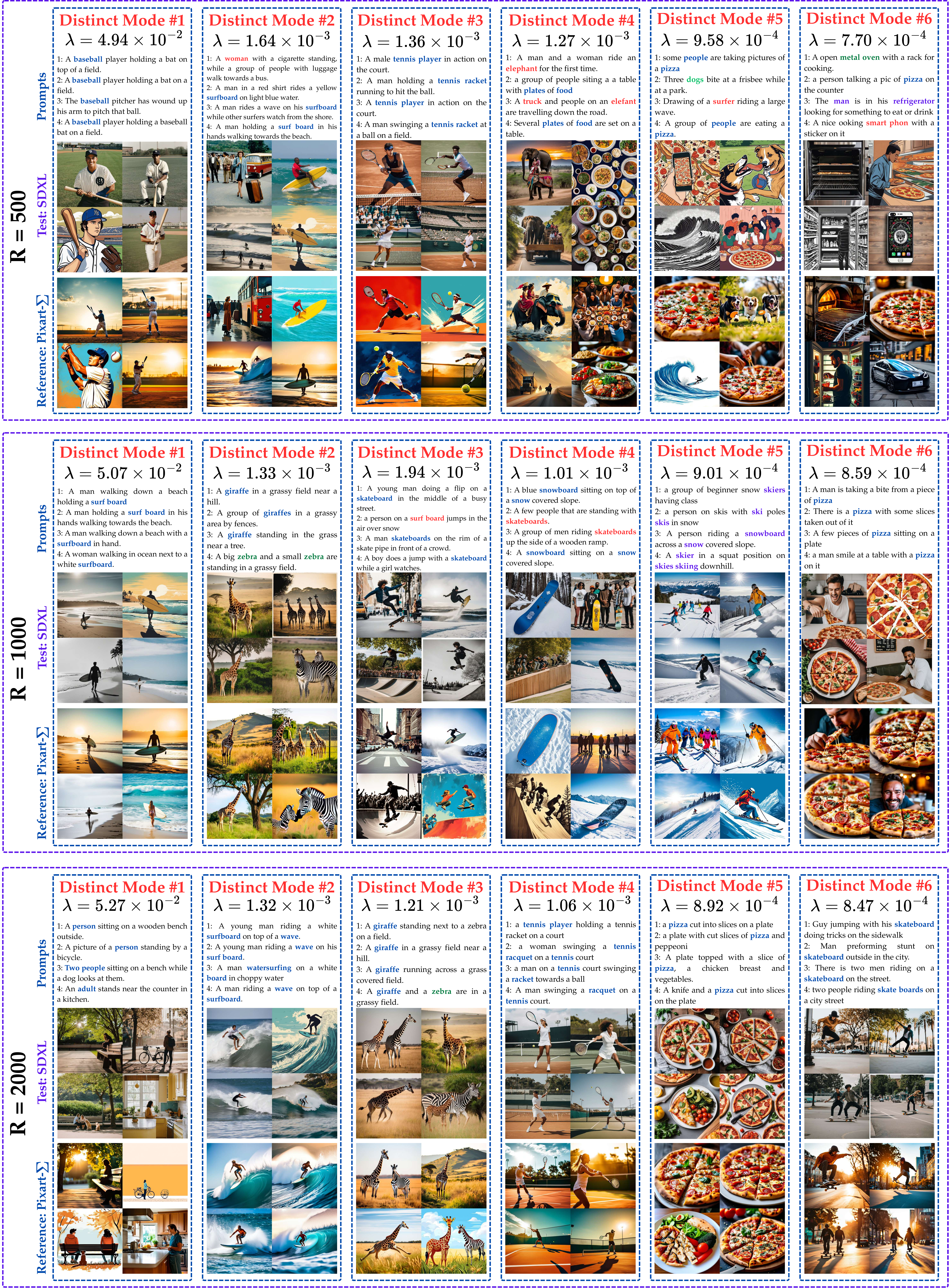}
    \caption{Top six modes identified by PromptSplit comparing SDXL - Pixart$\Sigma$ generatead images on MS-COCO with different Random Fourier features (RFF) r ranging from 1000 to 2000}
    \label{fig:r_ablation1}
\end{figure*}

\begin{figure*}
    \centering
    \includegraphics[width=0.95\linewidth]{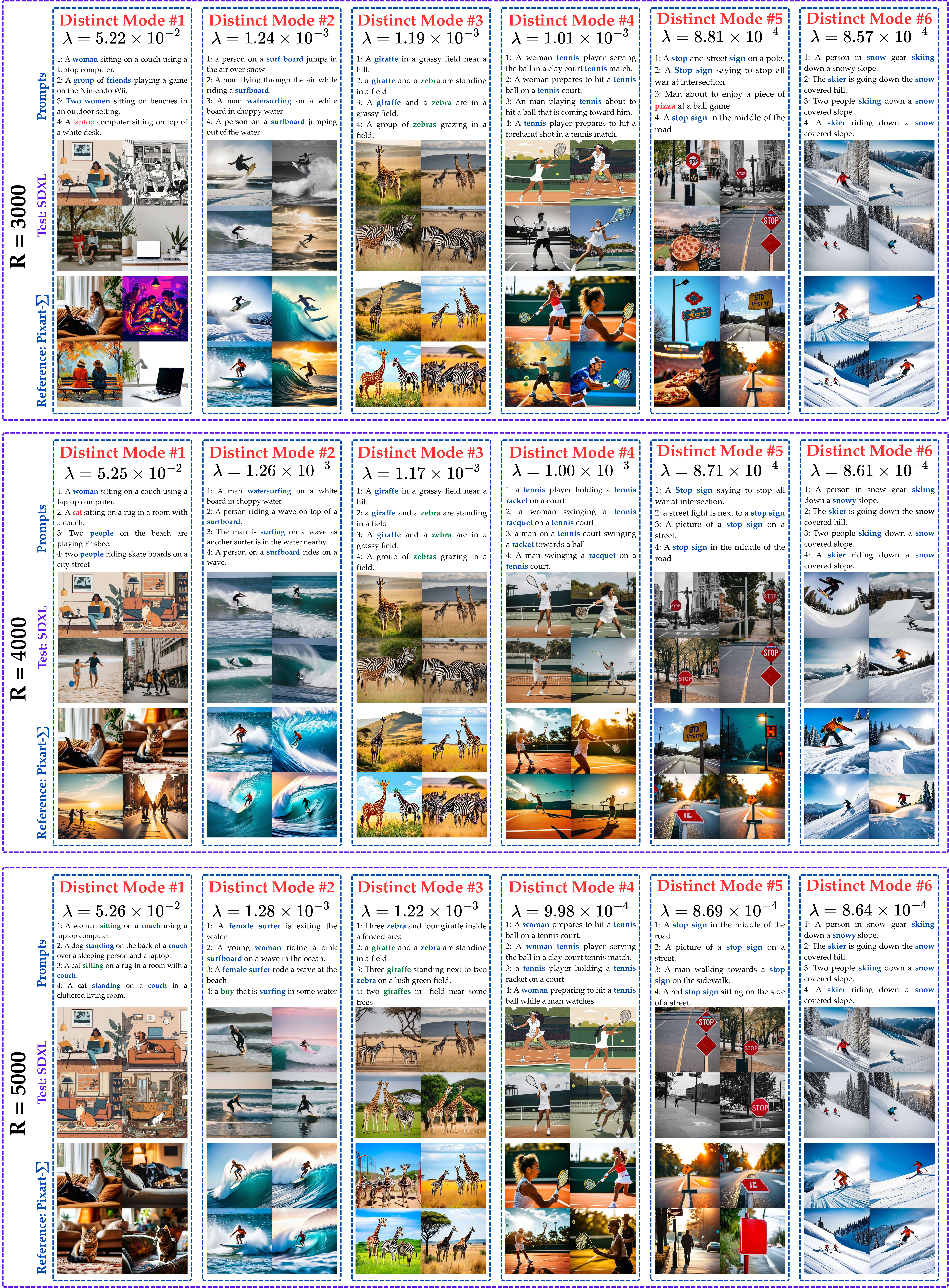}
    \caption{Top six modes identified by PromptSplit comparing SDXL - Pixart$\Sigma$ generatead images on MS-COCO with different Random Fourier features (RFF) r ranging from 3000 to 6000}
    \label{fig:r_ablation2}
\end{figure*}

\begin{figure*}
    \centering
    \includegraphics[width=0.95\linewidth]{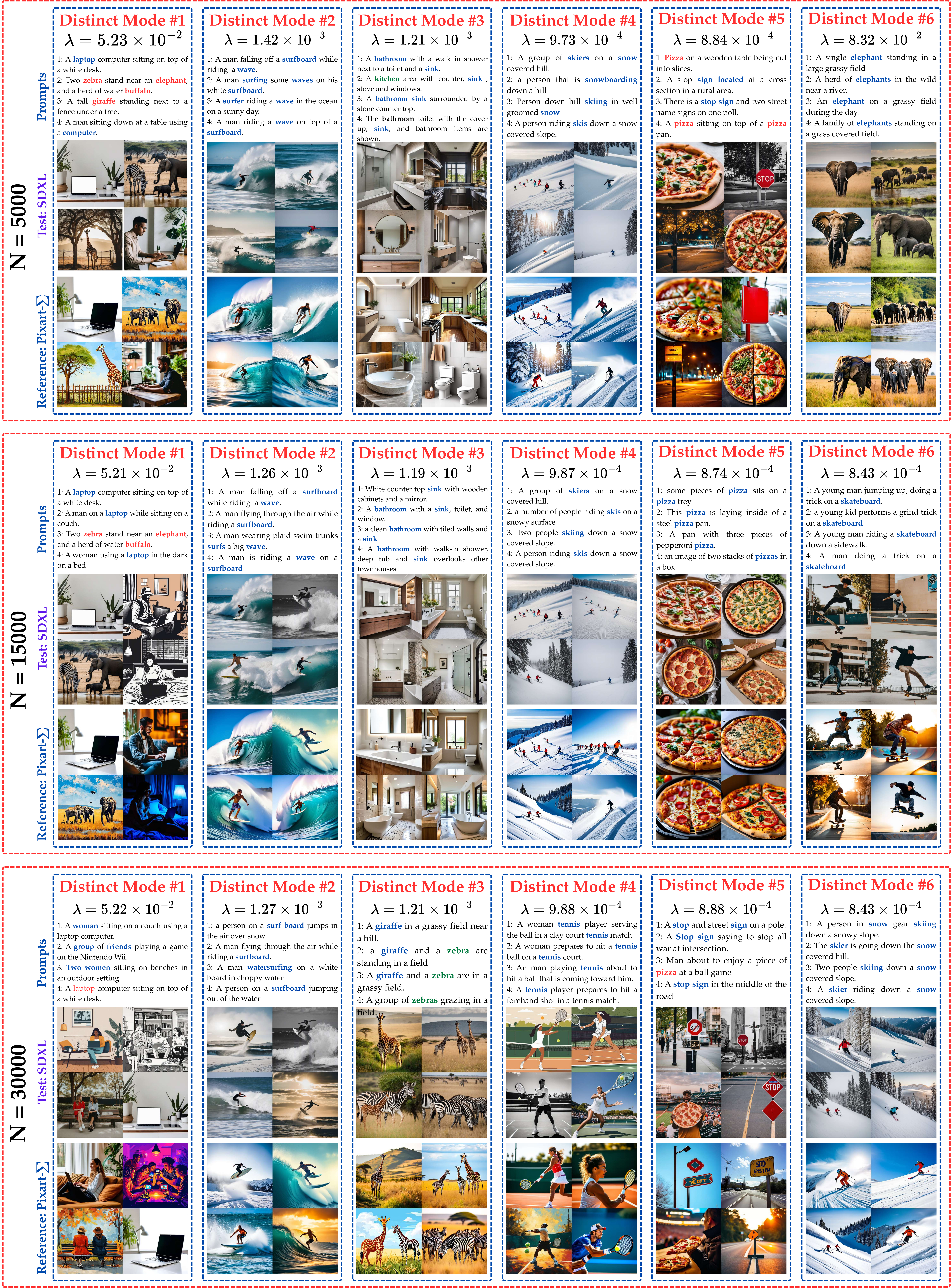}
    \caption{Top six modes identified by PromptSplit comparing SDXL - Pixart$\Sigma$ generatead images on MS-COCO with different number of samples n ranging from 5000 to 30000}
    \label{fig:n_ablation1}
\end{figure*}

\begin{figure*}
    \centering
    \includegraphics[width=0.95\linewidth]{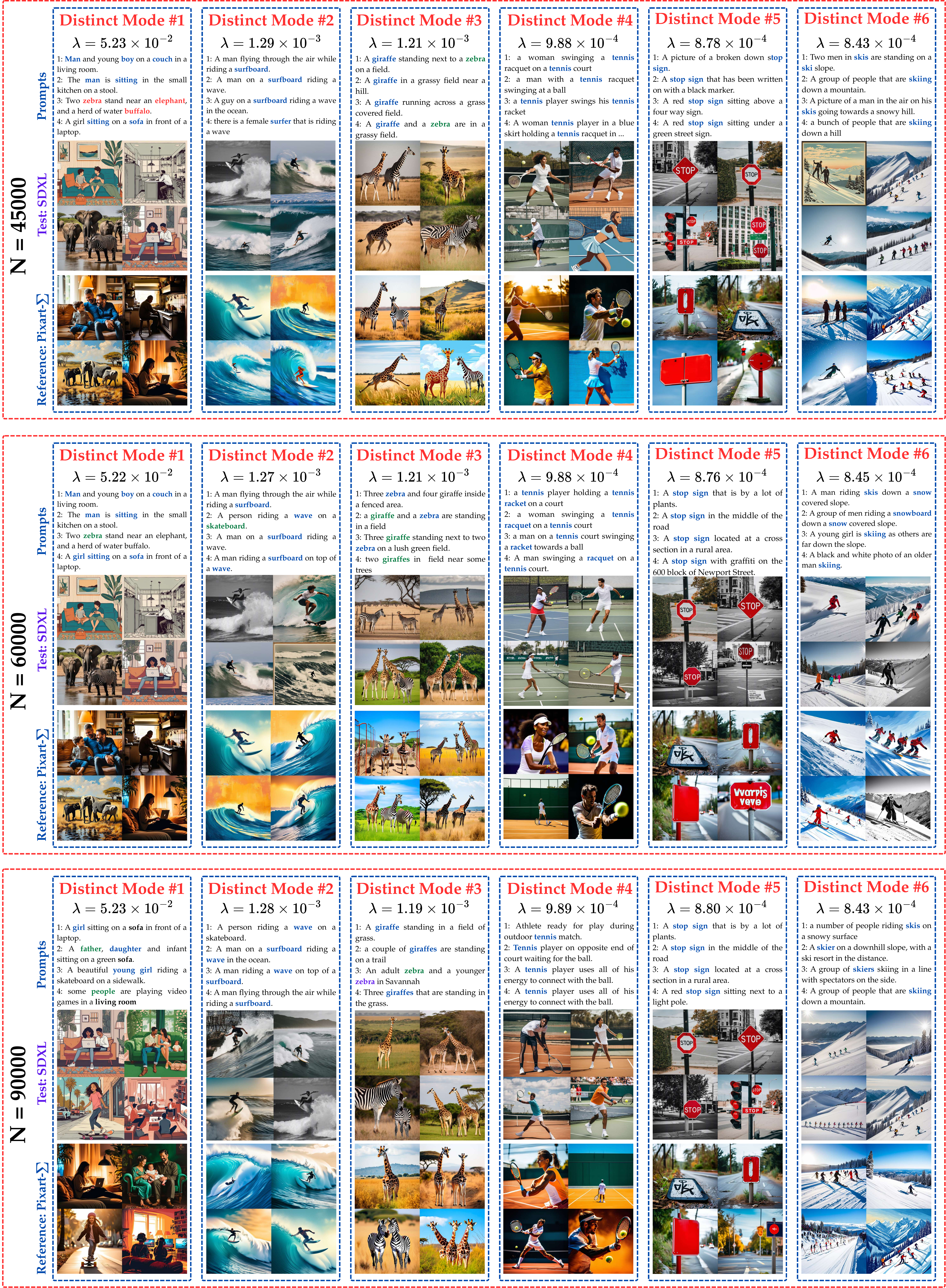}
    \caption{Top six modes identified by PromptSplit comparing SDXL - Pixart$\Sigma$ generatead images on MS-COCO with different number of samples n ranging from 45000 to 90000}
    \label{fig:n_ablation2}
\end{figure*}

\begin{figure*}[t]
    \centering
    \includegraphics[width=\linewidth]{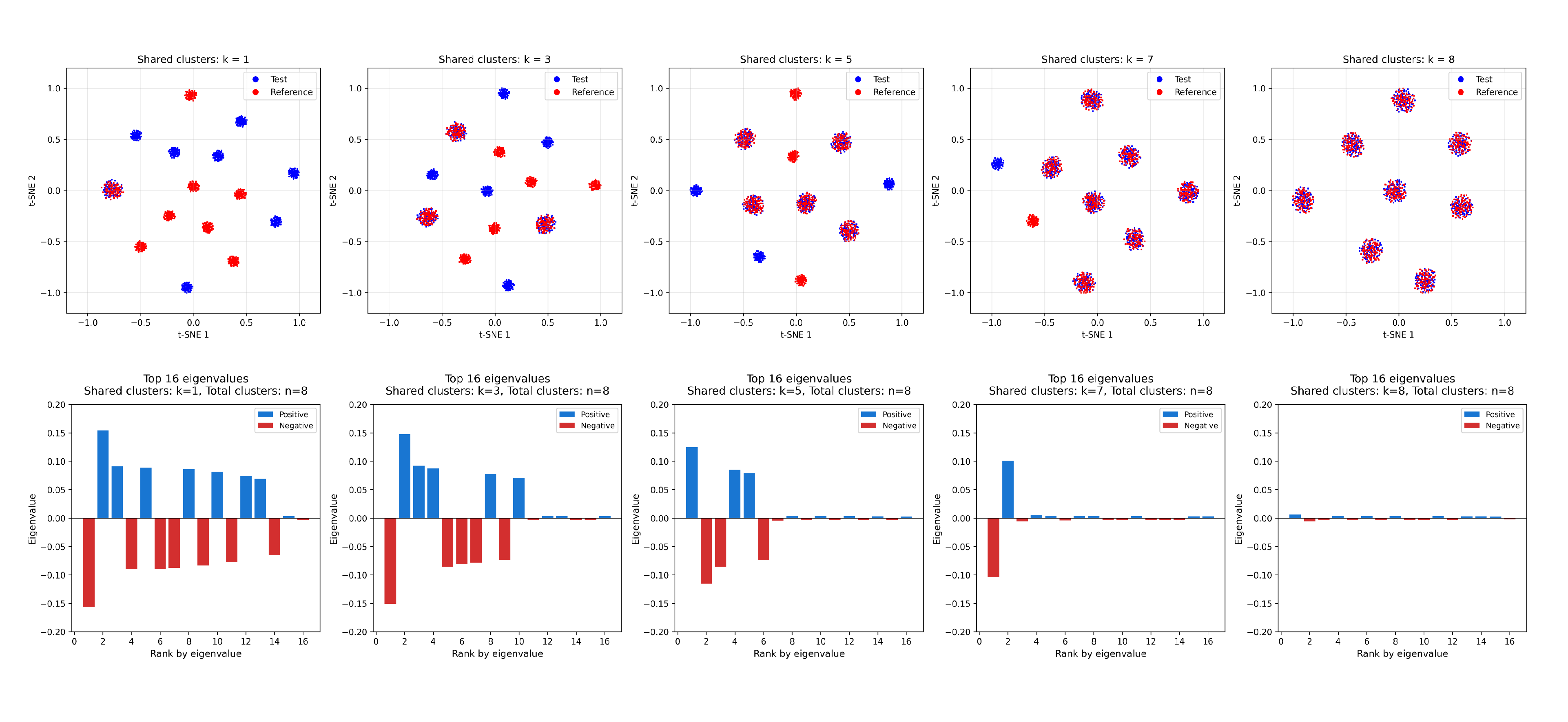}
    \caption{PromptSplit identified the numbers of distinct modes with significant large eigenvalues}.
    \label{fig:synthetic}
\end{figure*}

\begin{figure*}
    \centering
    \includegraphics[width=1\linewidth]{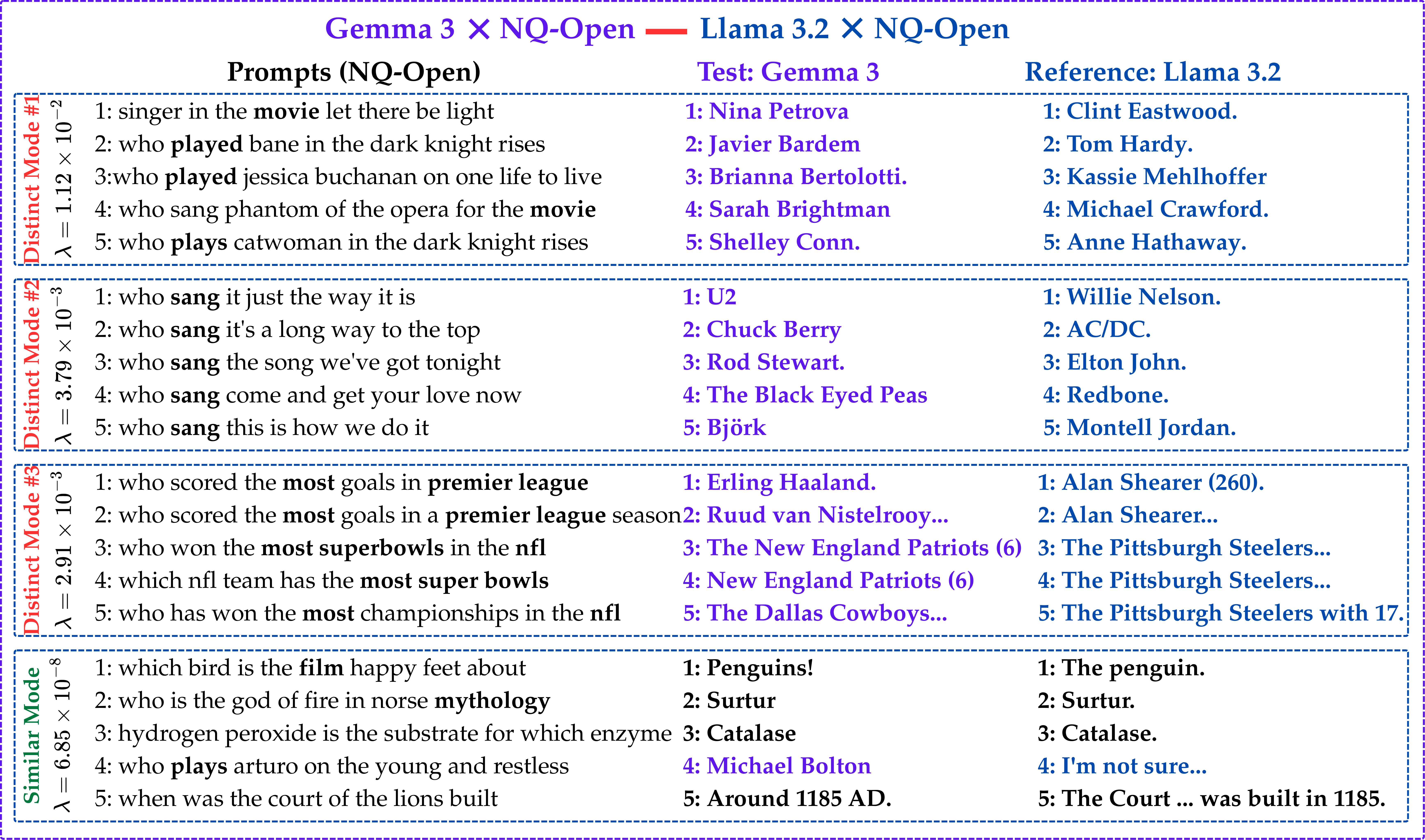}
    \caption{PromptSplit detected top distinct modes for 20000 generated short answers for test (Gemma3) w.r.t reference (Llama3.2) over NQ-Open questions as prompts.}
    \label{fig:gemma_llama_nq}
\end{figure*}


\begin{figure*}
    \centering
    \includegraphics[width=0.96\linewidth]{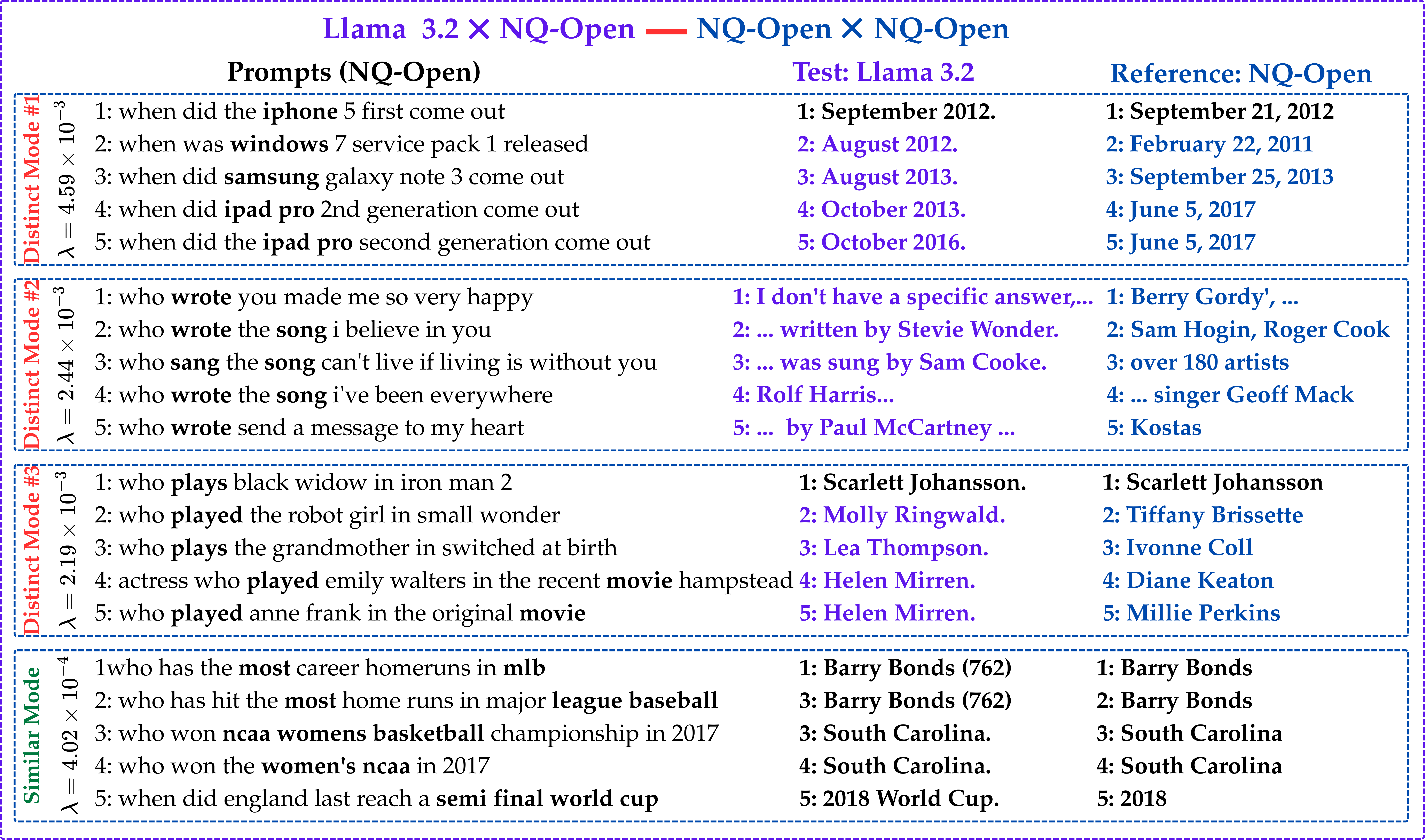}
    \caption{PromptSplit detected top distinct modes for 20000 generated short answers for test (Llama 3.2) w.r.t reference (NQ-Open Answers) over NQ-Open questions as prompts.}
    \label{fig:llama_nq_nq}
\end{figure*}

\begin{figure*}
    \centering
    \includegraphics[width=0.96\linewidth]{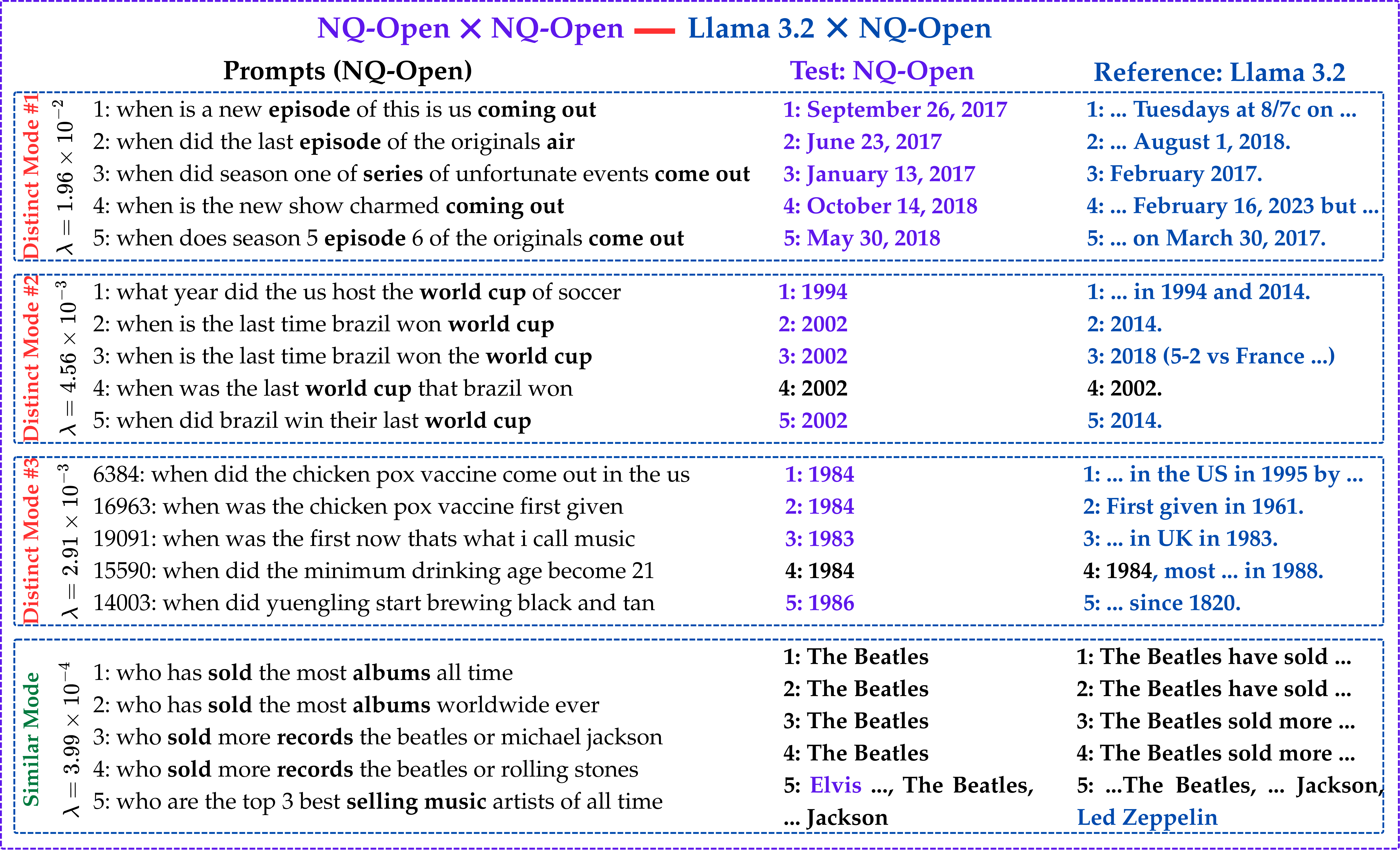}
    \caption{PromptSplit detected top distinct modes for 20000 short answers for test (NQ-Open) w.r.t generated answers for reference (Llama 3.2) over NQ-Open questions as prompts.}
    \label{fig:nq_llama_nq}
\end{figure*}

\begin{figure*}
    \centering
    \includegraphics[width=1\linewidth]{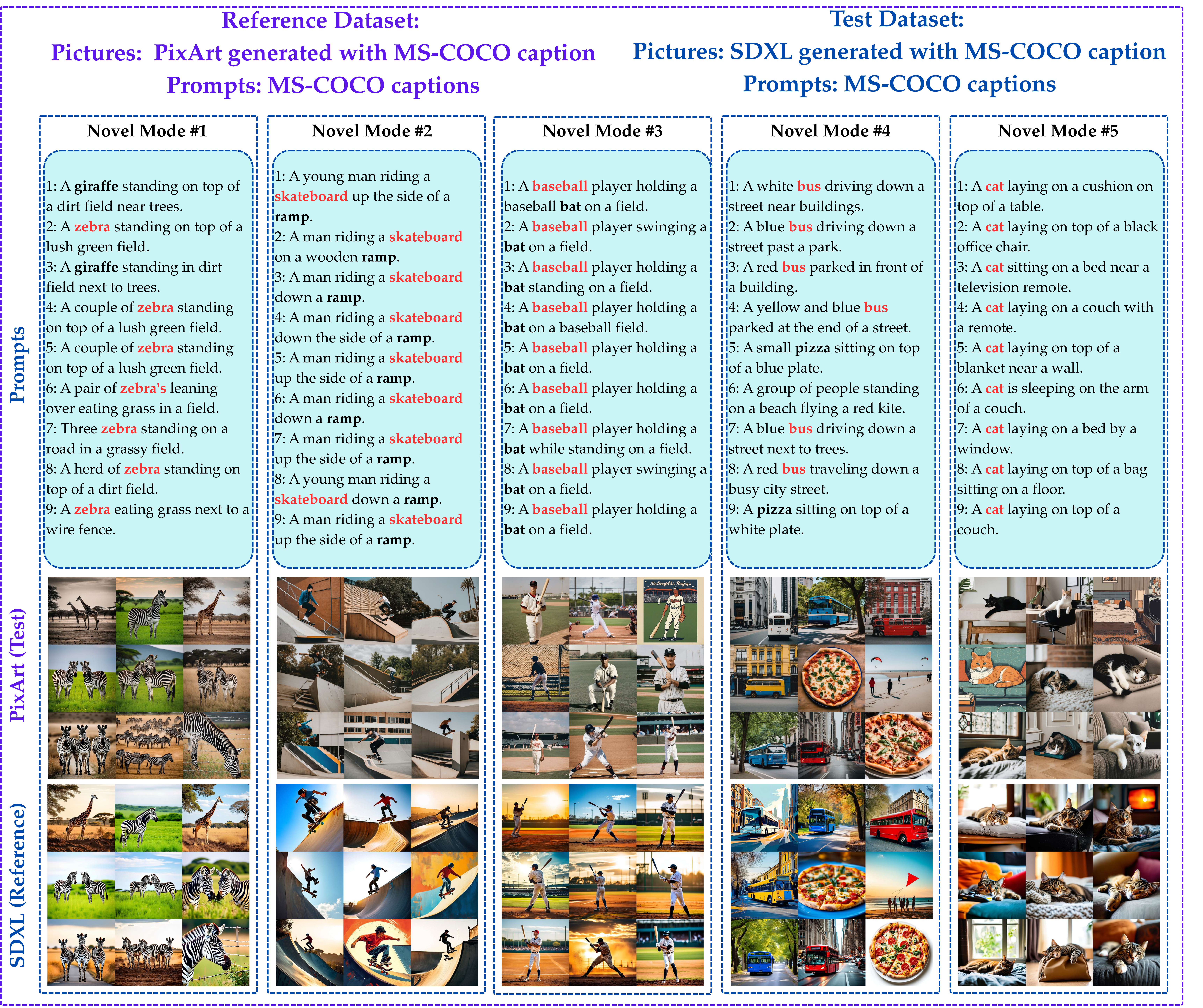}
    \caption{PromptSplit detected top novel modes for 30000 pictures generated by test (Pixart-$\Sigma$) w.r.t reference model (SDXL) over MS-COCO captions as prompts.}
    \label{fig:pixart_sdxl_mscoco}
\end{figure*}

\begin{figure*}
    \centering
    \includegraphics[width=1\linewidth]{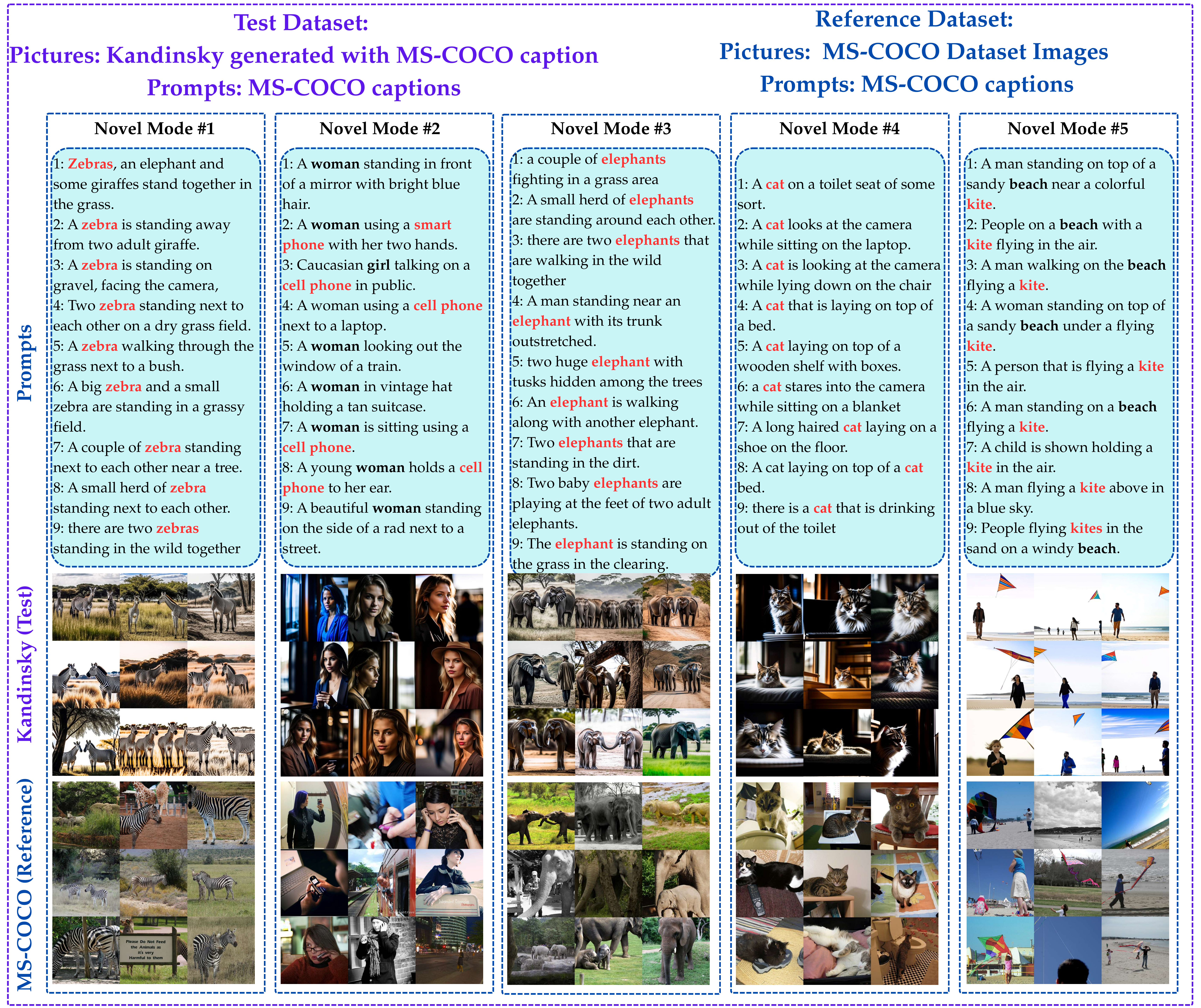}
    \caption{PromptSplit detected top novel modes for 30000 pictures generated by test (Kandinsky) w.r.t reference  (MS-COCO images) over MS-COCO captions as prompts.}
    \label{fig:kandinsky_mscoco_mscoco}
\end{figure*}

\begin{figure*}
    \centering
    \includegraphics[width=1\linewidth]{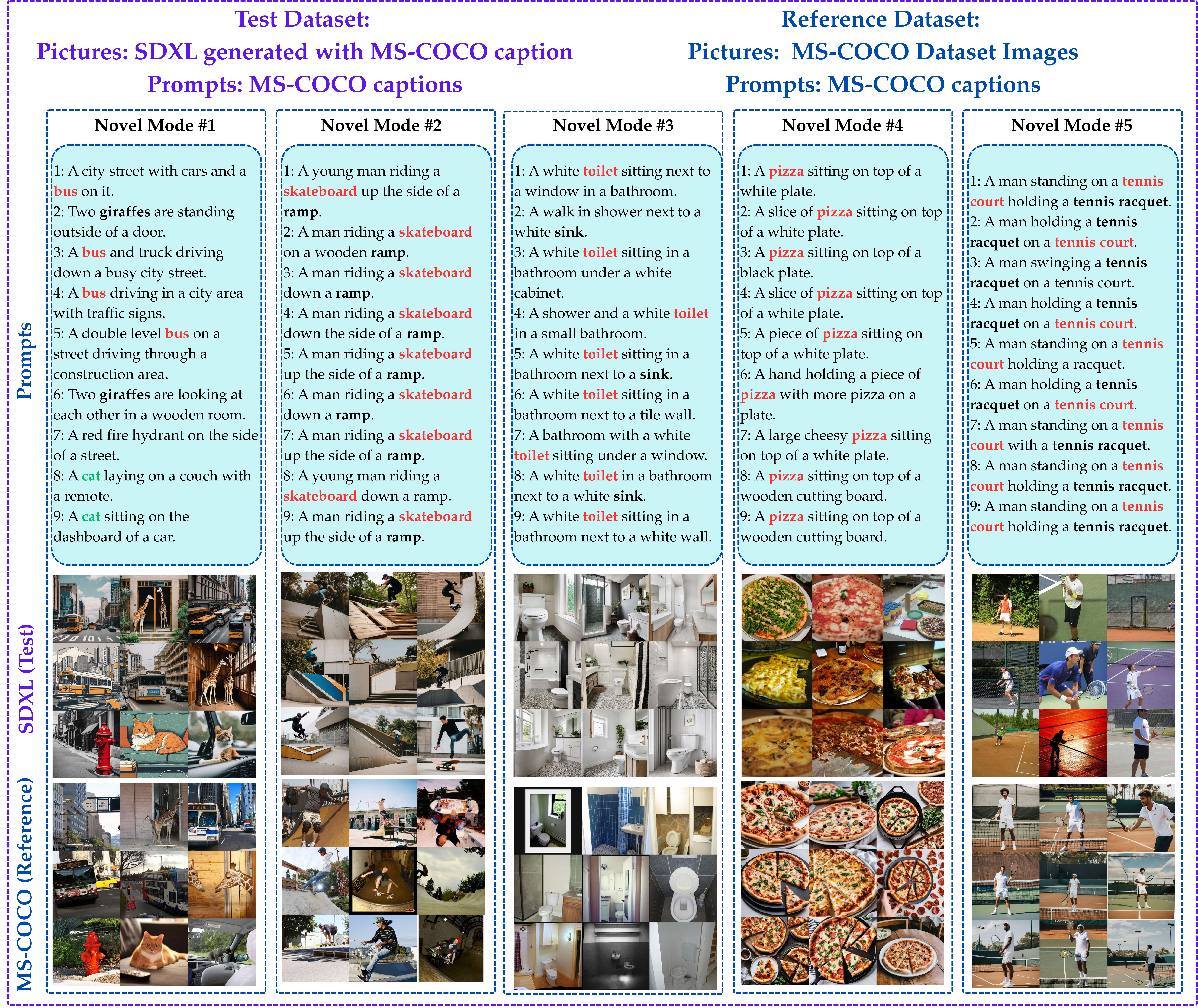}
    \caption{PromptSplit detected top novel modes for 30000 pictures generated by test (SDXL) w.r.t reference (MS-COCO images) over MS-COCO captions as prompts.}
    \label{fig:sdxl_mscoco_mscoco}
\end{figure*}

\begin{figure*}
    \centering
    \includegraphics[width=1\linewidth]{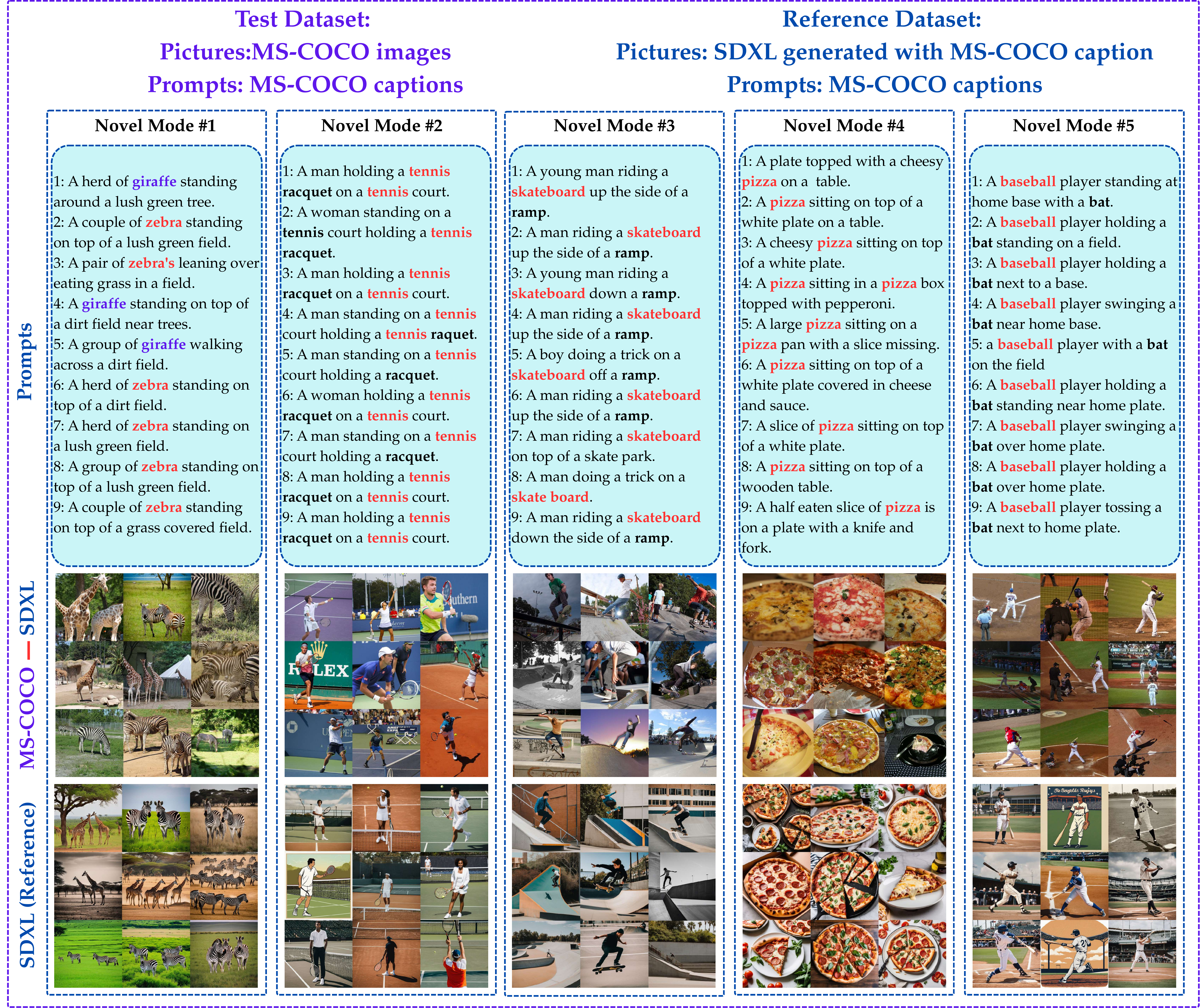}
    \caption{PromptSplit detected top novel modes for 30000 images for test (MS-COCO) w.r.t reference (SDXL generated) over MS-COCO captions as prompts.}
    \label{fig:mscoco_sdxl_mscoco}
\end{figure*}

\begin{figure*}
    \centering
    \includegraphics[width=1\linewidth]{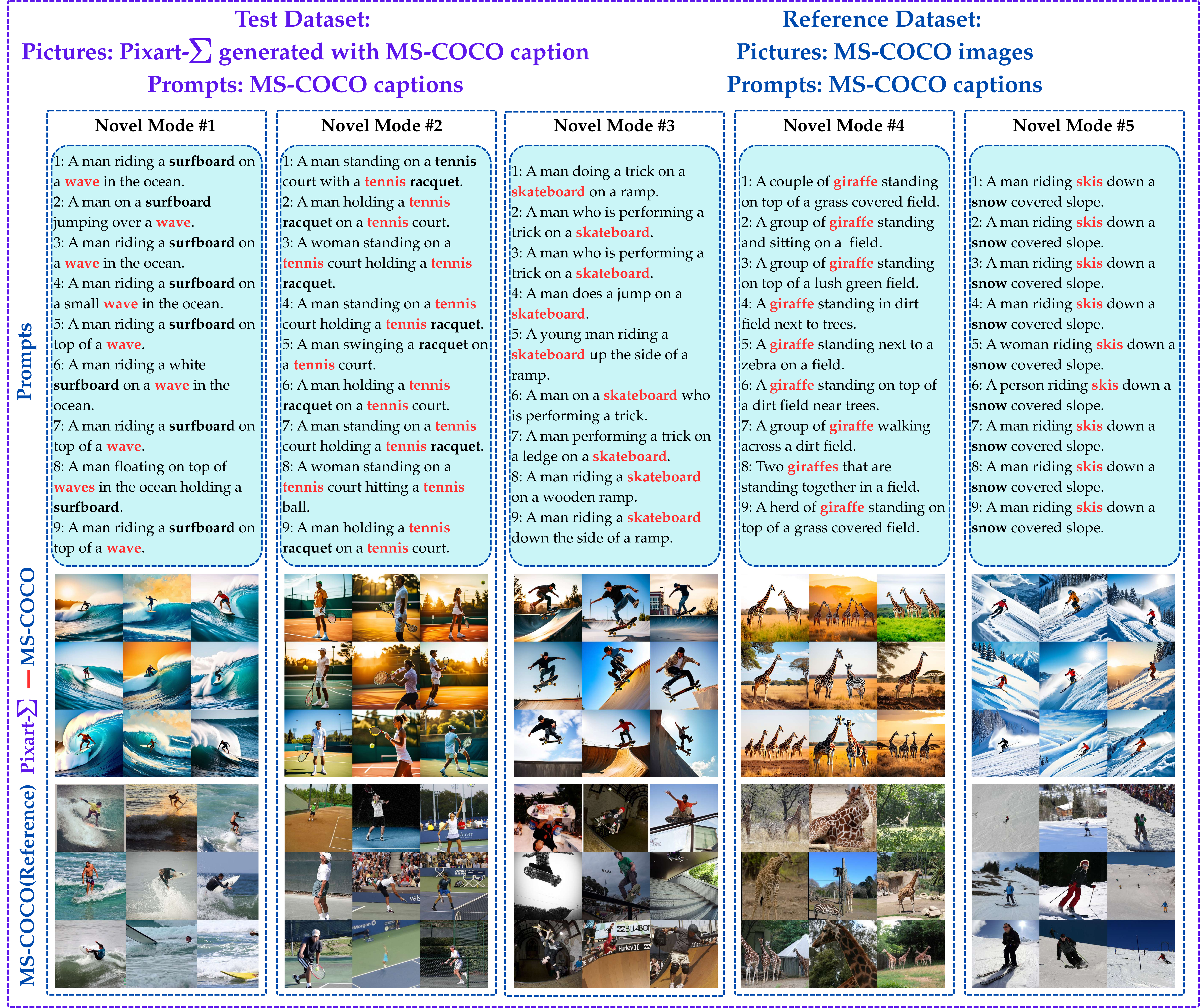}
    \caption{PromptSplit detected top novel modes for 30000 pictures generated by test (Pixart-$\Sigma$) w.r.t reference (MS-COCO images) over MS-COCO captions as prompts.}
    \label{fig:pixart_mscoco_mscoco}
\end{figure*}

\begin{figure*}
    \centering
    \includegraphics[width=1\linewidth]{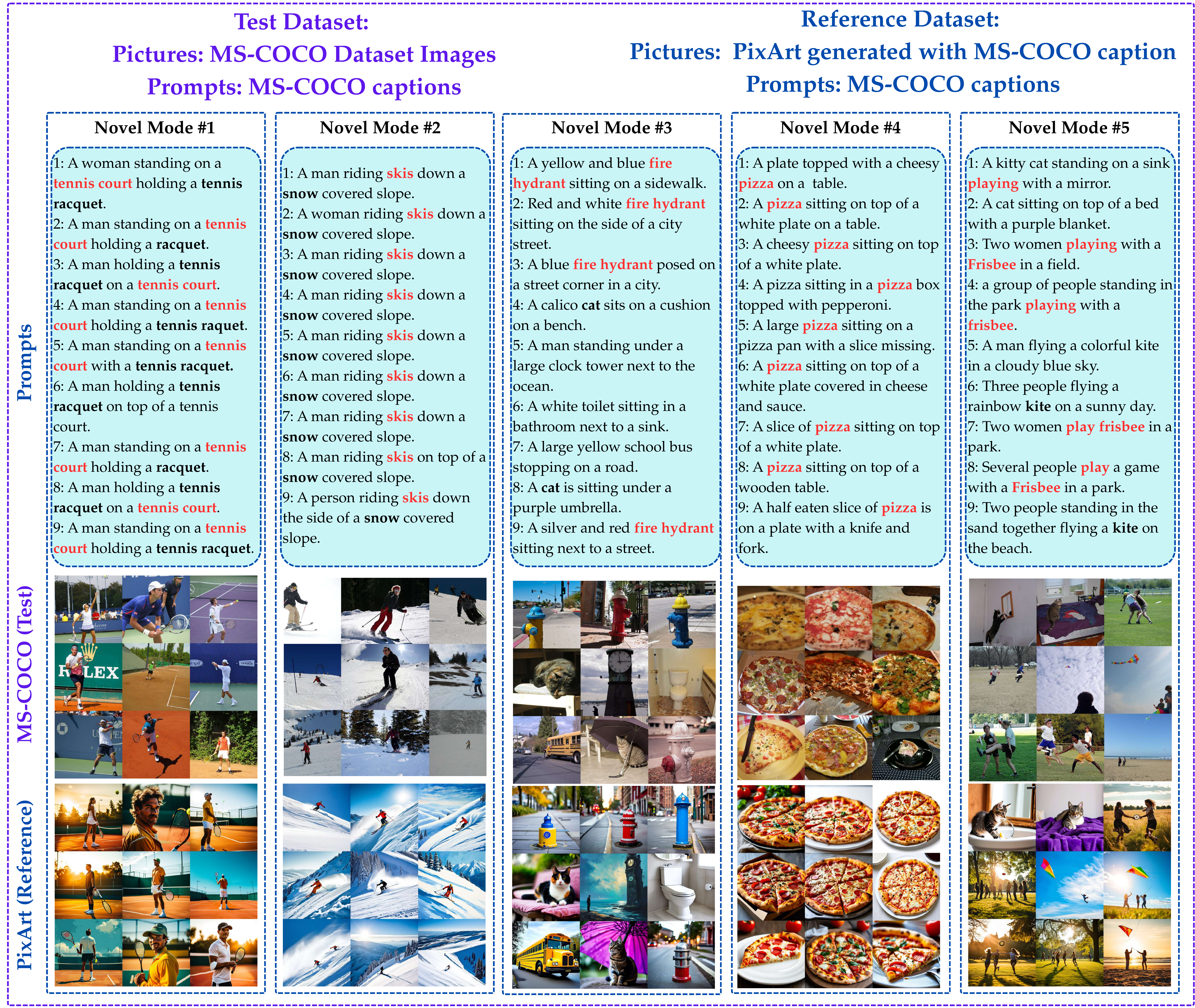}
    \caption{PromptSplit detected top novel modes for 30000 images for test (MS-COCO) w.r.t reference (Pixart-$\Sigma$ generated) over MS-COCO captions as prompts.}
    \label{fig:mscoco_pixart_mscoco}
\end{figure*}